\documentclass[10pt]{article} 
\usepackage[preprint]{rlc}

\usepackage{microtype}
\usepackage{graphicx}
\usepackage{subfigure}
\usepackage{booktabs} 
\usepackage{multirow}
\usepackage{multicol}
\usepackage{pifont}
\newcommand{\cmark}{\ding{51}}
\newcommand{\xmark}{\ding{55}}

\usepackage{algorithm}
\usepackage{algorithmic}

\usepackage{hyperref}
\usepackage{widebar}
\usepackage{rl-notations}
\usepackage{bbm}
\usepackage{xspace}

\usepackage{color, colortbl}

\definecolor{LightGray}{gray}{0.9}
\definecolor{LightCyan}{rgb}{0.88,1,1}

\newcommand{\proposal}{\hyperref[algo:primal-dual]{\texttt{UOpt-RPGPD}}\xspace}
\newcommand{\evaluate}{\hyperref[algo:policy evaluation]{\texttt{RegularizedPolicyEvaluation}}\xspace}

\usepackage{amssymb}            
\usepackage{mathtools}          
\usepackage{mathrsfs}           
\usepackage{graphicx}           
\usepackage{subcaption}         
\usepackage[space]{grffile}     
\usepackage{url}                
\usepackage{amsmath}
\usepackage{amsthm}

\usepackage[capitalize,noabbrev]{cleveref}

\theoremstyle{plain}
\newtheorem{theorem}{Theorem}[section]

\newtheorem{lemma}[theorem]{Lemma}
\newtheorem{corollary}[theorem]{Corollary}
\theoremstyle{definition}
\newtheorem{definition}[theorem]{Definition}
\newtheorem{assumption}[theorem]{Assumption}
\theoremstyle{remark}
\newtheorem{remark}[theorem]{Remark}

\usepackage{hyperref}  
\hypersetup{
  colorlinks,
  citecolor=blue,
  linkcolor=blue,
  urlcolor=blue
}

\crefname{algorithm}{Algorithm}{Algorithms}
\crefname{assumption}{Assumption}{Assumptions}
\crefname{corollary}{Corollary}{Corollaries}
\crefname{definition}{Definition}{Definitions}
\crefname{equation}{Equation}{Equations}
\crefname{example}{Example}{Examples}
\crefname{figure}{Figure}{Figures}
\crefname{lemma}{Lemma}{Lemmas}
\crefname{proposition}{Proposition}{Propositions}
\crefname{remark}{Remark}{Remarks}
\crefname{table}{Table}{Tables}
\crefname{theorem}{Theorem}{Theorems}

\Crefname{algorithm}{Algorithm}{Algorithms}
\Crefname{assumption}{Assumption}{Assumptions}
\Crefname{corollary}{Corollary}{Corollaries}
\Crefname{definition}{Definition}{Definitions}
\Crefname{equation}{Equation}{Equations}
\Crefname{example}{Example}{Examples}
\Crefname{figure}{Figure}{Figures}
\Crefname{lemma}{Lemma}{Lemmas}
\Crefname{proposition}{Proposition}{Propositions}
\Crefname{remark}{Remark}{Remarks}
\Crefname{table}{Table}{Tables}
\Crefname{theorem}{Theorem}{Theorems}

\title{A Policy Gradient Primal-Dual Algorithm for \\ Constrained MDPs with Uniform PAC Guarantees}

\author{
  Toshinori Kitamura 
  \thanks{
  Correspondence to: <toshinori-k@weblab.t.u-tokyo.ac.jp>.\\
  \textsuperscript{1} The University of Tokyo, Japan,
  \textsuperscript{2} OMRON SINIC X, Japan,
  \textsuperscript{3} Mizuho–DL Financial Technology, Japan,\\
  \textsuperscript{4} Nara Institute of Science and Technology, Japan,
  \textsuperscript{5} RIKEN Center for Advanced Intelligence Project, Japan.\\
  \textsuperscript{6} Osaka University, Japan.
  \textsuperscript{7} Kyoto University, Japan.
  }
  \textsuperscript{1},
  Tadashi Kozuno \textsuperscript{2,6},
  Masahiro Kato \textsuperscript{3,1},
  Yuki Ichihara \textsuperscript{4},\\
  Soichiro Nishimori \textsuperscript{1},
  Akiyoshi Sannai \textsuperscript{5,7},
  Sho Sonoda \textsuperscript{5,1},
  Wataru Kumagai \textsuperscript{1,5},
  Yutaka Matsuo \textsuperscript{1}
}

\begin{document}

\maketitle

\looseness=-1
\begin{abstract}
    We study a primal-dual (PD) reinforcement learning (RL) algorithm for online constrained Markov decision processes (CMDPs).
    Despite its widespread practical use, the existing theoretical literature on PD-RL algorithms for this problem only provides sublinear regret guarantees and fails to ensure convergence to optimal policies. 
    In this paper, we introduce a novel policy gradient PD algorithm with uniform probably approximate correctness (Uniform-PAC) guarantees, simultaneously ensuring convergence to optimal policies, sublinear regret, and polynomial sample complexity for any target accuracy.
    Notably, this represents the first Uniform-PAC algorithm for the online CMDP problem.
    In addition to the theoretical guarantees, we empirically demonstrate in a simple CMDP that our algorithm converges to optimal policies, while baseline algorithms exhibit oscillatory performance and constraint violation.
\end{abstract}

\section{Introduction}
\label{sec:submission}

\looseness=-1
This paper studies a primal-dual (PD) reinforcement learning (RL) algorithm for the online constrained Markov decision processes (CMDP) problem \citep{efroni2020exploration}, where the agent explores the environment with the aim of identifying an optimal policy that maximizes the return while satisfying certain constraints. 
The CMDP framework is particularly promising for designing policies in safety-critical decision-making applications, such as autonomous driving with collision avoidance \citep{He2023lane,Gu2023safe-state} and controlling thermal power plants with temperature satisfaction \citep{zhan2022deepthermal}.
Please refer to \citet{gu2022review} for more examples.

\begin{table*}[tb]
\footnotesize
\caption{
\looseness=-1
\footnotesize
Regret bound and $(\varepsilon, \delta)$-PAC bound comparison of online CMDP algorithms. The ``LP'' and ``PD'' rows correspond to linear programming and primal-dual algorithms, respectively. The ``Optimality'' and ``VIO'' columns correspond to the bounds for optimality gap and constraint violation, respectively. 
In the ``VIO'' column, ``same'' means that the bound for constraint violation is the same as that for the optimality gap, and ``const'' means that the bound does not depend on $K$.
In the ``Regret'' column, the subscript ``$\ _{+}$'' means that the bound is concerning to strong regret measures rather than weak measures (see \cref{sec:other performance measure}).
If the ``LIC?'' column is ``\cmark'', the algorithm is guaranteed the last-iterate convergence (LIC) to optimal policies.
This table is presented under single constraint settings (i.e., $N=1$) for a fair comparison.
The algorithms are: OptCMDP, OptPrimalDual-CMDP \citep{efroni2020exploration}, OptPress-LP, OptPress-PrimalDual \citep{liu2021learning}, DOPE \citep{bura2022dope}, OPDOP \citep{ding2021provably}, Triple-Q \citep{wei2021provably}, Online-CRL \citep{hasanzadezonuzy2021learning}, and Regularized Primal-Dual Algorithm \citep{muller2024truly}.
}
\label{tb:algortihm-comparisons}
\centering
\begin{tabular}{|c|c|cc|cc|c|}
\hline
\multirow{2}{*}{}   & \multirow{2}{*}{Algorithm} & \multicolumn{2}{c|}{Regret}                  & \multicolumn{2}{c|}{$(\varepsilon,\delta)$-PAC} & \multirow{2}{*}{\begin{tabular}[c]{@{}l@{}}LIC?\end{tabular}} \\ \cline{3-6}
                    &                            & \multicolumn{1}{c|}{Optimality} & VIO & \multicolumn{1}{c|}{Optimality}   & VIO  &                            \\ \hline
\multirow{4}{*}{LP} &  ${\text{ OptCMDP}}$                   & \multicolumn{1}{c|}{$\tiO\paren*{XA^{\frac{1}{2}}H^2K^{\frac{1}{2}}}_{+}$}           &  { $\text{same}_{+}$ }         & \multicolumn{1}{c|}{-}             &   -         &  \xmark                         \\ 
                    &  ${\text{ OptPress-LP}}$               & \multicolumn{1}{c|}{$\tiO\paren*{\bgap^{-1}X^{\frac{3}{2}}A^{\frac{1}{2}}H^3K^{\frac{1}{2}}}$}           &   { $0_{+}$ }        & \multicolumn{1}{c|}{-}             &    -        &      \xmark                     \\ 
                    &  ${\text{ DOPE}}$& \multicolumn{1}{c|}{$\tiO\paren*{\bgap^{-1}X A^{\frac{1}{2}}H^3K^{\frac{1}{2}}}$}           &  { $0_{+}$ }         & \multicolumn{1}{c|}{-}             &  -           &       \xmark                    \\ 
                    &  ${\text{ Online-CRL}}$                & \multicolumn{1}{c|}{-}           &     -      & \multicolumn{1}{c|}{$\tiO\paren*{X^{2}AH^2\varepsilon^{-2}}$}   &     same        &           \xmark                \\ \hline
\multirow{5}{*}{PD} &  ${\text{ OPDOP}}$& \multicolumn{1}{c|}{$\tiO\paren*{XA^{\frac{1}{2}}H^\frac{5}{2}K^{\frac{1}{2}}}$}           &   same         & \multicolumn{1}{c|}{-}             & -           &  \xmark                         \\ 
                    &  ${\text{ OptPD-CMDP}}$        & \multicolumn{1}{c|}{$\tiO\paren*{\bgap^{-1} X^{\frac{3}{2}}A^{\frac{1}{2}}H^2K^{\frac{1}{2}}}$}           &   same         & \multicolumn{1}{c|}{-}             & -           &  \xmark                         \\ 
                    &  ${\text{ OptPress-PD}}$       & \multicolumn{1}{c|}{$\tiO\paren*{\bgap^{-1}X^{\frac{3}{2}}A^{\frac{1}{2}}H^3K^{\frac{1}{2}}}$}           &  const         & \multicolumn{1}{c|}{-}             & -           &  \xmark                         \\ 
                    &  ${\text{ Triple-Q}}$                 & \multicolumn{1}{c|}{$\tiO\paren*{\bgap^{-1}X^{\frac{1}{2}}A^{\frac{1}{2}}H^4K^{\frac{4}{5}}}$}           &   0        & \multicolumn{1}{c|}{-}             & -           &  \xmark                         \\ 
                    &  ${\text{ Regularized PD}}$                 & \multicolumn{1}{c|}{$\tiO\paren*{\bgap^{-2}X^{\frac{1}{2}}A^{\frac{1}{4}}H^{\frac{9}{2}}K^{0.93}}_+$}           &   $\text{same}_+$        & \multicolumn{1}{c|}{-}             & -           &  \xmark                         \\ 
\rowcolor{LightGray}                    &  ${\text{ \textbf{UOpt-RPGPD}}}$ & \multicolumn{1}{c|}{$\tiO\paren*{\bgap^{-1} X^{\frac{1}{2}} A^{\frac{2}{7}}H^{4}K^{\frac{6}{7} }}_{+}$}           &   $\text{same}_{+}$         & \multicolumn{1}{c|}{$\tiO\paren*{\bgap^{-7}{X^{\frac{7}{2}} A^{2}H^{25}}\varepsilon^{-7}}$}             &   same          &     \cmark                       \\ \hline
\end{tabular}
\end{table*}

\looseness=-1
Two primary approaches for the online CMDP problem are the linear programming (LP) approach and the PD approach.
While the LP approach is common in theoretical literature \citep{efroni2020exploration,liu2021learning,bura2022dope,hasanzadezonuzy2021learning,zheng2020constrained}, the PD approach is more popular in practice due to its adaptability to high-dimensional problem settings.
The PD approach typically involves iterative policy gradient ascent over the Lagrange function, making it amenable to recent deep policy gradient RL algorithms \citep{achiam2017constrained,tessler2018reward,wang2022robust,le2019batch,russel2020robust}.

\looseness=-1
Despite its practical importance, theoretical results of PD RL algorithms are currently scarce.
Existing results on PD RL for online CMDPs are limited to sublinear regret guarantees \citep{efroni2020exploration,liu2021learning,wei2021provably,ding2021provably,gosh2023achiving}.
However, sublinear regret guarantees only bound the integral of the magnitude of mistakes during the training, and they cannot ensure the performance of the last-iterate policy up to arbitrary accuracy \citep{dann2017unifying}.
An alternative performance measure is $(\varepsilon, \delta)$-PAC, which ensures that the last-iterate policy's performance is sufficiently close to an optimal policy.
However, $(\varepsilon, \delta)$-PAC has only been established for LP algorithms in the online CMDP problem \citep{hasanzadezonuzy2021learning}.
Furthermore, it is known that both sublinear regret and $(\varepsilon, \delta)$-PAC are insufficient to ensure convergence to optimal policies \citep{dann2017unifying}.
CMDP algorithms lacking convergence guarantees may yield policies exhibiting oscillatory performance and constraint violation, those are undesirable in practical applications due to their potential impact on system stability and safety \citep{moskovitz2023reload}.

\looseness=-1
Numerous studies have tackled the convergence problem of PD algorithms for CMDPs, but most of them are limited to cases without exploration. Even in the absence of exploration, these studies possess unfavorable limitations for application, such as optimization over occupancy measures \footnote{An occupancy measure of a policy denotes the set of distributions over the state-action space generated by executing the policy in the environment. See \cref{eq:occupancy measure} for the definition.} rather than policies \citep{moskovitz2023reload}, providing convergence guarantees only through an average of past returns or a mixture of past policies \citep{li2021faster,chen2021primal,ding2020natural,liu2021policy}, and converging to a biased solution with fixed $\varepsilon > 0$ \citep{ying2022dual,ding2023last}.
More related works can be found in \cref{sec:related work}.
In light of these limitations, a natural question then arises:

\looseness=-1
\begin{center}
\emph{
Is it possible to design a policy gradient PD algorithm for online CMDPs that ensures the triplet of \\sublinear regret, $(\varepsilon, \delta)$-PAC, and convergence to optimal policies?}
\end{center}

\looseness=-1
We provide an affirmative response by proposing a novel policy gradient PD algorithm for online CMDPs with a uniform probably approximate correctness (Uniform-PAC) guarantee \citep{dann2017unifying}, called \textbf{U}niform-PAC \textbf{Opt}imistic \textbf{R}egularized \textbf{P}olicy \textbf{G}radient \textbf{P}rimal-\textbf{D}ual (\proposal). 
Uniform-PAC, a stronger performance metric than sublinear regret and $(\varepsilon, \delta)$-PAC, ensures not only convergence to optimal policies but also sublinear regret and polynomial sample complexity for any target accuracy.
Notably, \proposal is the first-ever online CMDP algorithm with Uniform-PAC guarantees for both LP and PD approaches.
Furthermore, \proposal ensures strong regret guarantees whereas most of the existing PD algorithms \citep{efroni2020exploration,ding2021provably,liu2021learning,wei2021provably} provide weaker guarantees as relying on the \emph{error cancellation} technique (see \cref{remark:performance measure} in \cref{sec:other performance measure} for details).
The very recent \citet{muller2024truly} provides a PD algorithm with a strong regret guarantee but they cannot be Uniform-PAC due to their bonus design (see \cref{subsec:classical-primal-dual}).
\cref{tb:algortihm-comparisons} compares the theoretical guarantees of online CMDP algorithms.


\looseness=-1
Finally, we empirically demonstrate the effectiveness of the three techniques through an ablation study on a simple CMDP.
Our results illustrate that \proposal converges to optimal policies, while other baseline algorithms fail to converge or even exhibit oscillatory behaviors.
\section{Preliminary}\label{sec:problem-setup}
\looseness=-1
We use the shorthand $\R_+ \df [0, \infty)$.
The set of probability distributions over a set $\bS$ is denoted by $\Delta(\bS)$.
For a positive integer $N \in \N$, we define $[N]\df \brace*{1,\dots, N}$.
All scalar operators and inequalities should be understood point-wise when used for vectors and functions.
For example, for functions $f, g, z: \X \to \R$, we express ``$f(x) \geq g(x)$ for all $x$'' as $f \geq g$ and ``$z(x) = f(x) + g(x)$ for all $x$'' as $z = f + g$. 
For $p_1, p_2 \in \Delta(\A)$ with $p_1, p_2 > \bzero$, we define the KL divergence between $p_1$ and $p_2$ as $\KL{p_1}{p_2}\df\sum_{a\in\A} p_1(a) \ln\frac{p_1(a)}{p_2(a)}$.
For $x, a, b \in \R$, we define a clipping function such that $\clip(x, a, b) = \min\brace*{\max\brace{x, a}, b}$.
For two positive sequences $\brace{a_n}$ and $\brace{b_n}$ with $n=1,2, \ldots$, we write $a_n=O\left(b_n\right)$ if there exists an absolute constant $C>0$ such that $a_n \leq C b_n$ holds for all $n \geq 1$.
We use $\widetilde{O}(\cdot)$ to further hide the polylogarithmic factors.

\looseness=-1
\paragraph{Constrained Markov Decision Processes.}
Let $N \in \brace{0, 1, \dots}$ be the number of constraints.
A finite-horizon and episodic CMDP is defined as a tuple $(\X, \A, H, P, \br, b, x_1)$, where
$\X$ denotes the finite state space with size $\aX$,
$\A$ denotes the finite action space with size $\aA$,
$H \in \N$ denotes the horizon of an episode,
$b \in [0, H]^N$ denotes the constrained threshold vector, where $b^n$ is the threshold scalar for the $n$-th constraint,
$x_1 $ denotes the fixed initial state,
and $\br\df\brace*{r^n}_{n\in \brace{0}\cup [N]}$ denotes the set of reward functions, where
$r^{n}_{\cdot}(\cdot, \cdot): \HXA \to [0, 1]$ denotes the $n$ th reward function,
and $r^n_{h}(x, a)$ is the $n$-th reward when taking an action $a$ at a state $x$ in step $h$.
The reward function $r^0$ is for the objective to optimize and the reward functions $\brace{r^1, \dots, r^N}$ are for constraints. 
$P_\cdot\paren*{\cdot\given \cdot, \cdot}: \HXA \to \Delta(\X)$ denotes the transition probability kernel, where $P_{h}\paren{y \given x, a}$ is the state transition probability to a new state $y$ from a state $x$ when taking an action $a$ in step $h$.
With an abuse of notation, for a function $V : \X \to \R$, let $P_h$ be an operator such that $(P_h V)(x, a) = \sum_{y\in \X} V(y) P_h \paren{y \given x, a}$.

\looseness=-1
\paragraph{Policy and Regularized Value Functions.}
A policy is defined as $\pi_{\cdot}\paren*{\cdot\given\cdot} : [H]\times\X \to \Delta(\A)$, where $\pi_h \paren{a \given x}$ denotes the probability of taking an action $a$ at state $x$ in step $h$.
For a policy $\pi$ with $\pi > \bzero$, we denote $\ln\pi_\cdot\paren{\cdot \given \cdot}: [H] \times \XA \to \R$ as the function such that $\ln\pi_h\paren{x\given a} = \ln\paren*{\pi_h \paren{a \given x}}$.
The set of all the policies is denoted as $\Pi$.
With an abuse of notation, for any policy $\pi $ and $Q: \XA \to \R$, let $\pi_h$ be an operator such that $(\pi_h Q) (x) = \sum_{a \in \A} \pi_h\paren{a\given x} Q (x, a)$.

\looseness=-1
For a policy $\pi$, transition kernel $P$, reward function $r_\cdot(\cdot, \cdot): \HXA \to \R$, and an entropy coefficient $\tau \geq 0$, let $\vf{\pi}_{\cdot}[P, r, \tau](\cdot): [H]\times\X \to \R$ be the regularized value function such that
\begin{equation*}
    \vf{\pi}_{h}[P, r, \tau](x)
    =
    \E \brack*{
        \sum^H_{i=h} r_h(x_{i}, a_{i}) - \tau \ln \pi_h \paren{a_i \given x_i}
        \given
        x_h=x, \pi, P
    }\;,
\end{equation*}
where the expectation is taken over all possible trajectories, in which $a_h \sim \pi_h \paren{\cdot \given x_h}$ and $x_{h+1} \sim P_h \paren{\cdot \given x_h, a_h}$.
We set $\vf{\pi}_{H+1}[P, r, \tau](x) = 0$ for all $x\in \X$.
We define the regularized action-value function $\qf{\pi}_\cdot[P, r, \tau](\cdot, \cdot): [H]\times \XA\to \R$ such that 
$$\qf{\pi}_h[P, r, \tau](x, a) = r_h(x, a) + \paren*{P_h\vf{\pi}_{h+1}[P, r, \tau]}(x, a)\;.$$
We set $\qf{\pi}_{H+1}[P, r, \tau](x, a) = 0$ for all $(x, a)\in \XA$.
When $\tau=0$, we omit $\tau$ from notations. For example, we write 
$\qf{\pi}_{h}[P, r] \df \qf{\pi}_{h}[P, r, 0]$.

\looseness=-1
\subsection{Learning Problem Setup}
We consider an algorithm operating an agent that repeatedly interacts with a CMDP $(\X, \A, H, P, \br, b, x_1)$ by playing a sequence of policies $\pi^1, \pi^2, \dots$, where $\pi^k \in \Pi$ denotes the policy that the agent follows in the $k$-th episode. 
Each episode $k$ starts from the fixed initial state $x_1$.
At the beginning of each time-step $h\in [H]$ in an episode $k$, the agent observes a state $x_h^k$ and chooses an action $a_h^k$, which is sampled from $\pi^k_h(\cdot\mid x_h^k)$.
The next state is sampled as $x_{h+1}^k \sim P_h\paren{\cdot\given x_h^k, a_h^k}$.
The learning agent lacks prior knowledge of the transition kernel $P$. 
We assume that the set of reward functions $\br$ is known for simplicity, but extending our algorithm to unknown stochastic rewards poses no real difficulty \citep{azar2017minimax,ayoub2020model}.

\looseness=-1
Let $\Pisafe \df \brace*{\pi \given \min_{n \in [N]} \paren*{\vf{\pi}_{1}[P, r^n](x_1) - b^n} \geq 0}$ be a set of policies that do not violate the constraints.
An optimal policy $\pi^\star$, which maximizes the non-regularized return while satisfying all the constraints, is defined as
\begin{equation*}
    \pi^\star \in \argmax_{\pi \in \Pisafe} \vf{\pi}_{1}[P, r^0](x_1)\;.
\end{equation*}

Finally, we assume the following slater condition. This assumption is mild as it holds when there exists some strictly feasible policy \citep{efroni2020exploration,liu2021learning,paternain2019constrained}.

\begin{assumption}[Slater Point]\label{assumption:slater}
There exists an unknown policy $\pisafe \in \Pisafe$ such that $\vf{\pisafe}_1[P, r^n](x_1) = \bsafe^n$, where $\bsafe^n > b^n$ for all $n \in [N]$. Let $\bgap \df \min_{n \in [N]} \paren*{ \bsafe^n - b^n }$.
\end{assumption}

\subsection{Performance Measure}

\looseness=-1
Let $\gapretm^k \df \vf{\pi^\star}_1[P, r^0](x_1) - \vf{\pi^k}_1[P, r^0](x_1)$ and $\gapviom^k \df \max_{n \in [N]} b^n - \vf{\pi^k}_1[P, r^n](x_1)$ be the temporal optimality gap and constraint violation, respectively.
Let $\gapret^k \df \max\brace*{\gapretm^k, 0}$ and $\gapvio^k \df \max\brace*{\gapviom^k, 0}$ be their positively clipped values.
For any $K \in \N$ and $\varepsilon > 0$, the following notations are useful to introduce the performance measures:
\begin{align}
\missret^\varepsilon \df \sum^{\infty}_{k=1}\mathbbm{1}\brace{\gapretm^k > \varepsilon},\;
\text{ and }\;
 \missvio^\varepsilon \df \sum^{\infty}_{k=1}\mathbbm{1}\brace{\gapviom^k > \varepsilon}\;.
\end{align}
$\missret^\varepsilon$ and $\missvio^\varepsilon$ measure the count of mistakes that exceed $\varepsilon > 0$ related to optimality gap and constraint violation, respectively.

\looseness=-1
The performance of an online CMDP algorithm is typically measured by either the high-probability regret \citep{efroni2020exploration,liu2021learning,bura2022dope} or $(\varepsilon, \delta)$-PAC \citep{hasanzadezonuzy2021learning,zeng2022finite,vaswani2022near,bennett2023provable}. 
Their formal definitions are deferred to \cref{sec:other performance measure}.
Since neither sublinear regret nor $(\varepsilon, \delta)$-PAC guarantees convergence to optimal policies \citep{dann2017unifying}, we consider the following Uniform-PAC measure to evaluate the algorithm's performance.

\begin{definition}[Uniform-PAC]\label{definition:uniform-PAC}
Given $\varepsilon > 0$ and $\delta \in (0, 1]$, let $F_{\mathrm{UPAC}} \paren{\cdots}$ be shorthand for $F_{\mathrm{UPAC}}\paren*{X, A, H, N, 1 / \bgap, 1 /\varepsilon, \ln\paren*{1 / \delta}}$, where $F_{\mathrm{UPAC}}$ is a real-valued function polynomial in all its arguments.
An algorithm achieves Uniform-PAC for $\delta > 0$ if there exists $F_{\mathrm{UPAC}}(\cdots)$ such that 
\begin{equation*}
\P\paren*{\exists \varepsilon > 0 \suchthat
   \missret^{\varepsilon} > F_{\mathrm{UPAC}}(\cdots)
    \;\vee\;
    \missvio^{\varepsilon} > F_{\mathrm{UPAC}}(\cdots)
} \leq \delta\;.
\end{equation*}
\end{definition}


\looseness=-1
\begin{theorem}\label{theorem:uniform-PAC-properties}
Suppose an algorithm is Uniform-PAC for $\delta$ with 
$F_{\mathrm{UPAC}}(\cdots)=\tiO\paren*{C\varepsilon^{-\alpha}}$, where $C, \alpha > 0$ are constants independent of $\varepsilon$.
Then, the algorithm 
\begin{enumerate}
    \item converges with high probability: $\P\paren*{\lim_{k\to \infty} \gapretm^k =0 \wedge\lim_{k\to\infty}\gapviom^k=0}\geq 1-\delta$.
    \item is $(\varepsilon, \delta)$-PAC with sample complexity $F_{\mathrm{UPAC}}(\cdots)$ for all $\varepsilon > 0$.
    \item 
     achieves $\tiO\paren*{C^{\frac{1}{\alpha}} K^{1-\frac{1}{\alpha}}}$ regret with probability at least $1-\delta$, where $K \in \N$. 
\end{enumerate}
\end{theorem}
The first and the second claims are direct applications of \textbf{Theorem 3} from \citet{dann2017unifying}.
The third claim follows from \cref{lemma:error-to-regret} in \cref{appendix:useful-lemma}.

\section{The UOpt-RPGPD Algorithm}\label{sec:proposal}

\looseness=-1
This section provides our \textbf{U}niform-PAC \textbf{Opt}imistic \textbf{R}egularized \textbf{P}olicy \textbf{G}radient \textbf{P}rimal-\textbf{D}ual (\proposal) algorithm. 
Our \proposal relies on the combination of three key techniques:
$(\mathrm{i})$ the Lagrange function regularized by policy entropy and Lagrange multipliers,
$(\mathrm{ii})$ Uniform-PAC exploration bonus,
and $(\mathrm{iii})$ careful adjustment of the regularization coefficient and learning rate.
We present the pseudo-code of our algorithm in \cref{algo:primal-dual}.
It is important to remark that existing online CMDP algorithms are designed only for a fixed iteration length $K \in [N]$ \citep{efroni2020exploration,liu2021learning,bura2022dope,hasanzadezonuzy2021learning,zheng2020constrained,wei2021provably,wei2022provably,ding2021provably,amani2021safe,gosh2023achiving}. 
In contrast, our algorithm works with an infinite episode length.

\looseness=-1
\paragraph{Regularized Lagrange function.}
\proposal is designed to solve the following regularized Lagrange function in \cref{eq:regularized-Lagrange} while exploring the environment.
For a policy $\pi \in \Pi$, Lagrange multipliers $\lambda \in \R_+^N$, and an entropy coefficient $\tau \geq 0$, we define the regularized Lagrange function as 
\begin{equation}\label{eq:regularized-Lagrange}
L_\tau(\pi, \lambda) 
\df 
\vf{\pi}_{1}\brack*{P, r^0, \tau}(x_1)
+ \sum_{n=1}^N\lambda^n \paren*{\vf{\pi}_{1}\brack*{P, r^n}(x_1) - b^n}
+ \frac{\tau}{2} \|\lambda\|_2^2\;.
\end{equation}

\looseness=-1
Let $\pi^\star_{\tau} \df \argmax_{\pi \in \Pi} \min_{\lambda \in \R^N_+} L_\tau (\pi, \lambda)$ and 
$\lambda^\star_{\tau} \df \argmin_{\lambda \in \R^N_+} \max_{\pi\in \Pi} L_\tau (\pi, \lambda)$.
Note that $(\pi^\star_\tau, \lambda^\star_\tau)$ is the unique saddle point of $L_\tau$ as the following lemma shows (the proof is provided in \cref{sec:proof of strong-duality}).
\begin{lemma}\label{lemma:strong duality}
For any $\tau \in (0, \infty)$, there exists a unique saddle point $(\pi^\star_\tau, \lambda^\star_\tau) \in \Pi\times \R_+^N$ such that 
$$
L_\tau(\pi^\star_\tau, \lambda) \geq 
L_\tau(\pi^\star_\tau, \lambda^\star_\tau) \geq 
L_\tau(\pi, \lambda^\star_\tau)
\;\; \forall (\pi, \lambda) \in \Pi\times \R^N_+
\;.$$
\end{lemma}

\looseness=-1
The regularized Lagrange technique is derived from the work by \citet{ding2023last}, wherein the consideration of exploration is absent.
The introduced regularization affords us to upper bound the value of 
$\sum^H_{h=1}\sum_{x\in \X} w_h^{\pi_\tau^\star}[P](x)\KL{\pi^\star_{\tau_k, h}\paren*{\cdot \given x}}{\pi^k_h\paren*{\cdot \given x}}$ 
by a decreasing function on $k$.
Intuitively, it implies that the pair $(\pi^k, \lambda^k)$ converges to $(\pi^\star_{\tau_k}, \lambda^\star_{\tau_k})$, leading to the decreasing optimality gap and constraint violation of \proposal.
The detailed upper bound is provided in \cref{appendix:optimality-gap-and-vio-analysis}.

\begin{algorithm}[t!]
    \caption{UOpt-RPGPD}
    \begin{algorithmic}
    \STATE {\bfseries Input:} $\delta \in (0, 1]$, $\alpha_\eta \in (0, 1]$, $\alpha_\tau \in (0, 1]$ 
    \STATE Set $\lambda^1 \df \bzero$. Set $\pi^1_h(a\mid x) \df \frac{1}{\aA} \quad \forall (x, a, h) \in \HXA$
    \STATE Set $\eta_k \df (k+3)^{-\alpha_\eta}$ and $\tau_k \df (k+3)^{-\alpha_\tau}$
    \FOR{$k = 1, 2, 3, \dots $}
        \STATE Compute bonus $\beta^{k,\delta}$ by \cref{eq:bonus function}
        \STATE $\tqf{k,0}, \tvf{k,0} \df \evaluate(r^0, \beta^{k,\delta}, \barP^k, \pi^k, \tau_k)$ 
        \STATE $\tqf{k,n}, \tvf{k,n} \df \evaluate(r^n, \beta^{k,\delta}, \barP^k, \pi^k, 0)$ for all $n \in [N]$
        \STATE $\tqf{k} \df \tqf{k,0} + \sum^N_{n=1} \lambda^{k,n} \tqf{k,n}$
        \STATE Compute $\pi^{k+1}$ by \cref{eq:policy-mirror-ascent} and compute $\lambda^{k+1}$ by \cref{eq:Lagrange-update}
        \STATE Rollout $\pi^{k+1}$ and then update $n^k$ and $\barP^k$
    \ENDFOR
    \end{algorithmic}\label{algo:primal-dual}
\end{algorithm}

\begin{algorithm}[t!]
    \caption{RegularizedPolicyEvaluation}
    \begin{algorithmic}
    \STATE {\bfseries Input:} $r$, $\beta$, $\barP$, $\pi \in \Pi$, $\tau \in \R_+$
    \STATE Set $\widetilde{V}_{H+1} \df \bzero$
    \FOR{$h = H, \cdots 1$}
        \STATE $\widetilde{Q}_h \df \min\brace*{r_h + \paren*{1 + \tau\ln\aA}H\beta_h + \barP_h\widetilde{V}_{h+1},\;\paren*{1 + \tau \ln\aA} (H-h+1)\bone}$
        \STATE $\widetilde{V}_h \df {\pi_h\paren*{\widetilde{Q}_h - \tau\ln \pi_h}}$
    \ENDFOR
    \STATE {\bfseries Return:} $\widetilde{Q}$, $\widetilde{V}$
    \end{algorithmic}\label{algo:policy evaluation}
\end{algorithm}

\looseness=-1
\paragraph{Uniform-PAC Exploration Bonus.}
The second technique in our algorithm is the use of the Uniform-PAC bonus function.
Let $\llnp(x) \df \ln\paren*{\ln\paren*{\max\brace*{x, e}}
}$.
Let $n_h^k(x, a)\df\sum_{i=1}^k \mathbbm{1} \paren*{ x_h^{i}=x, a_h^{i}=a }$ be the number of times a pair $(x, a)$ was observed at step $h$ before episode $k+1$.
We define the empirical estimation of the transition model as
$$
\barP_h^{k}\left(y \mid x, a\right)\df\frac{\sum_{i=1}^{k} \mathbbm{1}\left(x_h^{i}=x, a_h^{i}=a, x_{h+1}^{i}=y\right)}{n_h^{k}(x, a) \vee 1}\;.
$$
Given a failure probability $\delta$, we define the bonus function $\beta^{k, \delta}_{\cdot}(\cdot, \cdot): \HXA\to \R$ such that 
\begin{equation}\label{eq:bonus function}
\begin{aligned}
\beta_{h}^{k, \delta}(x, a)
= \sum_{y \in \X} \sqrt{4\barP^k_h\paren{y\given x, a}}\phi + 5\phi^2 \; \text{ where }\;
\phi \df \sqrt{\frac{2 \llnp\paren*{2 n^k_h(x, a)}+2\ln \paren*{48\aX^2\aA H \delta^{-1}}}{n^k_h(x, a) \vee 1}}\;.
\end{aligned}
\end{equation}

\looseness=-1
Using the bonus $\beta^{k,\delta}$, \proposal computes the regularized optimistic value functions $\tqf{k, 0}, \tvf{k, 0}$, and the non-regularized optimistic value functions $\brace*{\tqf{k, n}, \tvf{k, n}}_{n \in [N]}$ by a regularized policy evaluation (\cref{algo:policy evaluation}).

\looseness=-1
Our bonus design is inspired by the work of \citet{dann2017unifying}.
While naive bonus functions (e.g., \citet{efroni2020exploration}) scale to $\sqrt{\ln(K)/{n^k_h(x, a)}}$ with a fixed iteration length $K$,
our bonus scales to $\sqrt{\ln \ln (n^k_h(x,a)) / n^k_h(x, a)}$.
This allows the bonus to diminish sufficiently quickly even when $K\to \infty$, in contrast to existing bonuses that can increase when $K$ becomes large.

\looseness=-1
\paragraph{Adjust Regularization Coefficients and Learning Rate.}
The combination of the above two techniques may not be sufficient for Uniform-PAC because it can introduce a biased solution due to the regularization in \cref{eq:regularized-Lagrange}. 
To overcome this issue, we decrease the regularization coefficient and the policy learning rate as $\eta_k \df k^{-\alpha_\eta}$ and $\tau_k \df k^{-\alpha_\tau}$ with $0 < \alpha_\tau < 0.5 < \alpha_\eta < 1$.
We remark that employing naive learning rates such as $\eta_k = \tau_k \propto 1 / k$ or  $\eta_k = \tau_k \propto 1 / \sqrt{k}$, as seen in prior works like \citet{efroni2020exploration}, fails to guarantee Uniform-PAC. 
To attain Uniform-PAC, we applied careful sequential analysis techniques akin to those utilized in bandit studies (e.g., \citet{cai2023uncoupled}) to our primal-dual CMDP algorithm. 

\looseness=-1
Coupled with this adjustment technique, \proposal updates the policy and the Lagrange multipliers based on the regularized Lagrange objective (\cref{eq:regularized-Lagrange}) with the Uniform-PAC bonus (\cref{eq:bonus function}).
Specifically, it updates the policy through an entropy-regularized natural policy gradient ascent \citep{cen2022fast} as
\begin{equation}\label{eq:policy-mirror-ascent}
\begin{aligned}
\pi^{k+1}_h(\cdot \mid x) \propto \paren*{\pi^k_h(\cdot \mid x)}^{(1-\eta_k\tau_k)} \exp\paren*{\eta_k\tqf{k}_h(x, \cdot)}
\end{aligned}
\end{equation}
for all $(x, h) \in \X \times [H]$, 
where $\tqf{k}$ is defined as $\tqf{k} \df \tqf{k,0} + \sum^N_{n=1} \lambda^{k,n} \tqf{k,n}$.

\looseness=-1
\proposal then updates the Lagrange multipliers through a projected regularized gradient descent, given by
\begin{equation}\label{eq:Lagrange-update}
\lambda^{k+1,n} \df \clip\brack*{\Lambda,\; 0,\; \frac{H(1+\tau_k\ln A)}{\bgap}} \quad \forall n \in [N]
\end{equation}
where $\Lambda \df \lambda^{k,n} + \eta_k\paren*{b^n - \tvf{k,n}_1(x_1) - \tau_k \lambda^{k,n}}$ and $\lambda^{k, n}$ denotes the $n$-th value of $\lambda^k$.

\begin{figure*}[tb!]
    \centering
    \includegraphics[width=\linewidth]{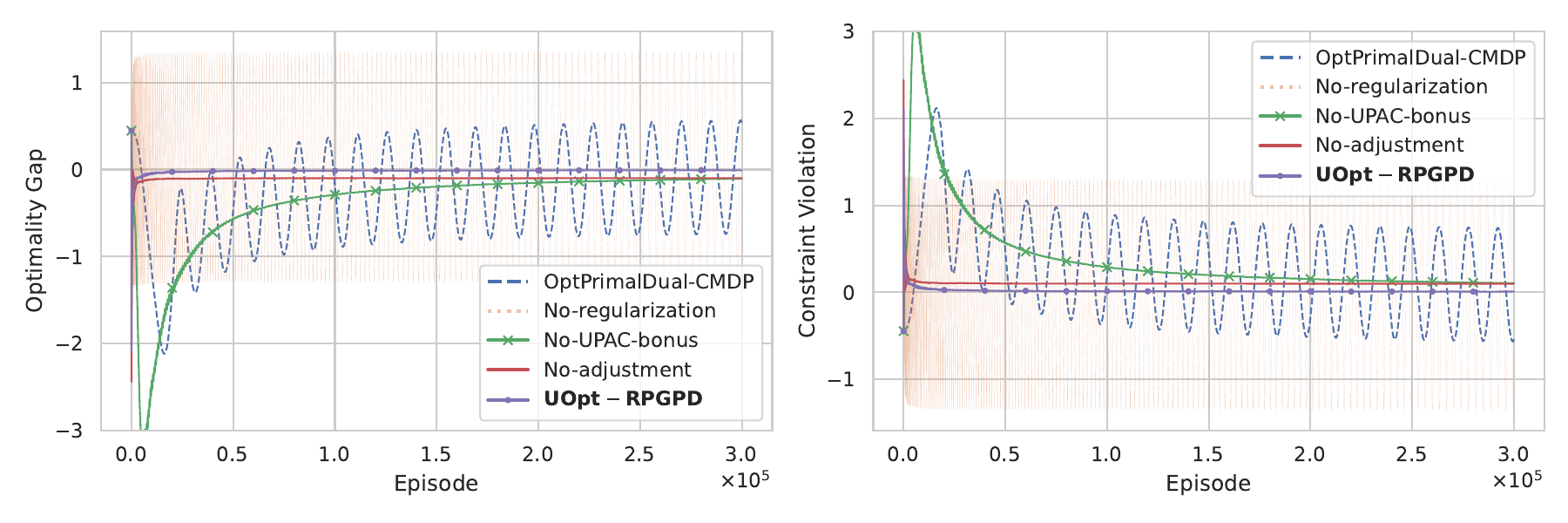}
    \caption{
    Comparison of the algorithms described in \cref{sec:experiments}.
    \textbf{Left}: optimality gap ($\gapretm^k$) and \textbf{Right}: constraint violation ($\gapviom^k$).
    }
    \label{fig:experiment}
\end{figure*}

\section{Uniform-PAC Analysis}
Our \proposal achieves the following Uniform-PAC guarantee.

\looseness=-1
\begin{theorem}\label{theorem:algorithm-is-uniform-PAC}
Suppose that \cref{assumption:slater} holds.
Set the regularization coefficient $\alpha_\tau$ and the learning rate $\alpha_\eta$ such that 
$0 < \alpha_\tau < 0.5 < \alpha_\eta < 1$ and 
$\alpha_\eta + \alpha_\tau < 1$.
Then, for $\delta > 0$, \proposal achieves Uniform-PAC for $F_{\text{UPAC}}(\cdots)$ such that
\begin{equation}\label{eq:order-tmp}
\begin{aligned}
F_{\text{UPAC}}(\cdots) 
&=
\underbrace{\tiO\paren*{\paren*{\bgap^{-2}\paren*{1+N}{XA^{\frac{1}{2}}H^7\varepsilon^{-2}}}^{\frac{1}{\alpha_\eta - 0.5}}}}_{\mathrm{(i)}}\\
&\quad+ \underbrace{\tiO\paren*{\paren*{\bgap^{-1}H\varepsilon^{-1}}^{\frac{1}{\alpha_\tau}}}}_{\mathrm{(ii)}}
+ \underbrace{\left(\frac{24}{1-(\alpha_\eta + \alpha_\tau)} \ln \frac{12}{1-(\alpha_\eta + \alpha_\tau)}\right)^{\frac{1}{1-(\alpha_\eta + \alpha_\tau)}}}_{\mathrm{(iii)}}\;.
\end{aligned}
\end{equation}
\end{theorem}

The proof is provided in \cref{sec:uniform-PAC-proof}.
Note that this bound depends on the values of $\alpha_\tau$ and $\alpha_\eta$.
Below, we establish a concrete bound by setting the values of $\alpha_\tau$ and $\alpha_\eta$, focusing on the order of $\varepsilon$.
The first term $\mathrm{(i)}$ decreases as $\alpha_\eta$ approaches 1.
However, the conditions $0 < \alpha_\tau < 0.5 < \alpha_\eta < 1$ and $\alpha_\eta + \alpha_\tau < 1$ restrict $\alpha_\eta$ from nearing $1$ due to the second term $\mathrm{(ii)}$ and the third term $\mathrm{(iii)}$. This tradeoff makes achieving $\tiO(\varepsilon^{-6})$ in \cref{eq:order-tmp} unattainable.

Hence, we select values for $\alpha_\eta$ and $\alpha_\tau$ to make the order of $\varepsilon$ in \cref{eq:order-tmp} to be $\tiO(\varepsilon^{-7})$. 
The ensuing corollary is presented as follows.

\looseness=-1
\begin{corollary}\label{corollary:uniform-PAC}
Suppose that \cref{assumption:slater} holds.
With $\alpha_\eta = \frac{11}{14}$ and $\alpha_\tau = \frac{1}{7}$, \proposal achieves Uniform-PAC for $\delta > 0$ for $F_{\text{UPAC}}$ such that
\begin{equation*}
F_{\text{UPAC}}(\cdots)=
\tiO\paren*{{\bgap^{-7}\paren*{1+N}^{\frac{7}{2}}X^{\frac{7}{2}} A^{2}H^{25}\varepsilon^{-7}}}\;.
\end{equation*}
Therefore, with probability at least $1-\delta$, \proposal converges to optimal policies and has regret
\begin{equation*}
    \tiO\paren*{\bgap^{-1}
    (1+N)^{\frac{1}{2}}X^{\frac{1}{2}} A^{\frac{2}{7}}H^{4}
    K^{\frac{6}{7}
    }}\;.
\end{equation*}
\end{corollary}

\looseness=-1
In \cref{corollary:uniform-PAC}, the convergence and regret bound follows immediately from \cref{theorem:uniform-PAC-properties}.
To our knowledge, \proposal is the first RL algorithm for online CMDPs that achieves the triplet of sublinear regret, $(\varepsilon, \delta)$-PAC, and convergence to optimal policies.

\section{Experiments}\label{sec:experiments}

\looseness=-1
This section describes the experimental behavior of our \proposal on a simple CMDP.
We randomly instantiate a CMDP with $X=30$, $A=3$, $H=10$, and $N=1$.
The construction of a CMDP is based on a tabular (C)MDP experiment conducted by \citet{dann2017unifying} and \citet{moskovitz2023reload}. 
A detailed description of the experimental setup can be found in \cref{sec:experiment details}.

\looseness=-1
In this experiment, we compare \proposal with \texttt{OptPrimalDual-CMDP} from \citet{efroni2020exploration} as it adheres to the naive primal-dual framework. The detail of the algorithm is provided in \cref{subsec:classical-primal-dual}.
Furthermore, to empirically validate the efficacy of the three techniques introduced in \proposal as expounded in \cref{sec:proposal}, we compare:
$(\mathrm{i})$ \proposal without regularization technique (i.e., $\tau_k = 0$),
$(\mathrm{ii})$ \proposal with the naive bonus function by \cref{eq:classical-bonus},
and $(\mathrm{iii})$ \proposal with fixed $\eta_k$ and $\tau_k$.
We call the three algorithms \texttt{No-regularization}, \texttt{No-UPAC-bonus}, and \texttt{No-adjustment}, respectively.

\looseness=-1
\cref{fig:experiment} compares algorithms and presents their optimality gap (Left) and constraint violation (Right).
The results are from a single run of the same randomly generated CMDP, yet it is illustrative.
We reran the experiment with different random seeds, consistently obtaining qualitatively similar results.

\looseness=-1
Compared to other algorithms, our \proposal quickly converges to optimal policies.
Algorithms without regularization, such as \texttt{No-regularization} and \texttt{OptPrimalDual-CMDP}, fail to converge and even display oscillatory behaviors. 
\texttt{No-UPAC-bonus} results in slow learning and \texttt{No-adjustment} converges to a biased solution.
These results highlight the importance of the three techniques introduced in \proposal.
\section{Conclusion}

\looseness=-1
We introduced \proposal, the first primal-dual RL algorithm for online CMDPs with Uniform-PAC guarantees.
\proposal ensures simultaneous convergence to optimal policies, sublinear regret, and polynomial sample complexity for any target accuracy.
In addition to the theoretical analysis, we empirically demonstrated that our algorithm successfully converges to optimal policies in a simple CMDP, whereas the existing primal-dual algorithm exhibits oscillatory behavior.


\looseness=-1
\paragraph{Limitation and Future Work.}
Although \proposal achieves Uniform-PAC, it may violate constraints during exploration. 
The development of a zero-violation algorithm is currently a significant topic in the study on online CMDP \citep{liu2021learning,bura2022dope,wei2021provably}. 
We defer the extension of our results to a zero-violation algorithm as part of our future work.

\looseness=-1
Another future research involves the extension to function approximation. 
A theoretical study of the primal-dual approach with function approximation could reveal opportunities for improvement in existing practical primal-dual deep RL algorithms. 
While there are CMDP algorithms with linear function approximation \citep{ding2021provably, amani2021safe, gosh2023achiving}, none establish Uniform-PAC guarantees when using function approximation. 
The extension to function approximation, such as linear MDPs \citep{zhou2021nearlyMinimax, hu2022nearly, he2023nearly} or general function approximation \citep{jiang2017contextual, du2021bilinear, jin2021bellman}, represents another promising direction for future work.

\looseness=-1
Finally, our Uniform-PAC bound may not be tight.
Compared to the $\tiO(\sqrt{K})$ regret bound by the LP algorithm \citet{efroni2020exploration}, our \cref{corollary:uniform-PAC} provides $\tiO(K^{\frac{6}{7}})$ regret bound.
It remains unclear whether this is an artifact of our analysis or a genuine limitation of Uniform-PAC primal-dual algorithms.
We leave this topic as a future work.

\looseness=-1
\subsubsection*{Broader Impact Statement}
\label{sec:broaderImpact}
This paper presents work whose goal is to advance the field of RL theory. There are many potential societal consequences of our work, none of which we feel must be specifically highlighted here.

\looseness=-1
\subsubsection*{Acknowledgments}
\label{sec:ack}
SS was supported by JST PRESTO JPMJPR2125.
WK was partially supported by the JSPS KAKENHI Grant Numbers 19H04071 and 23H04974.
AS was supported by JSPS KAKENHI Grant Number JP20K03743, JP23H04484 and JST PRESTO JPMJPR2123.


\bibliography{main}
\bibliographystyle{rlc}

\appendix

\section{Related Work}\label{sec:related work}

\looseness=-1
\paragraph{Online CMDP Algorithms.}
\looseness=-1
The seminal work by \citet{efroni2020exploration} provides both LP and primal-dual algorithms with sublinear regret. \citet{liu2021learning} and \citet{bura2022dope} extend this work, achieving sublinear performance regret with a zero constraint violation guarantee during learning. \citet{wei2021provably} propose a model-free primal-dual algorithm with sublinear regret, and \citet{wei2022provably} extend it to the average-reward setting.
\citet{ding2021provably}, \citet{amani2021safe}, and \citet{gosh2023achiving} propose CMDP algorithms with linear function approximation and sublinear regret guarantees.
While the sublinear regret guarantee is common in online CMDPs, even an optimal regret algorithm may still make infinitely many mistakes of arbitrary quality (e.g., \textbf{Theorem 2} in \citet{dann2017unifying}).

\looseness=-1
While numerous online CMDP algorithms guarantee sublinear regret, those equipped with $(\varepsilon, \delta)$-PAC guarantees remain scarce.
\citet{hasanzadezonuzy2021learning} present an LP algorithm with an $(\varepsilon, \delta)$-PAC guarantee for online CMDPs but do not include a primal-dual algorithm.
\citet{zeng2022finite} and \citet{vaswani2022near} offer primal-dual algorithms with $(\varepsilon,\delta)$-PAC guarantees for infinite horizon CMDPs. 
However, \citet{zeng2022finite} assume that the MDP is ergodic, and \citet{vaswani2022near} assume access to a simulator, both potentially obscuring the challenges associated with exploration.
\citet{bennett2023provable} provide an $(\varepsilon, \delta)$-PAC algorithm for the problem wherein the state-action-wise constraint signal is given by binary feedback.
Although these $(\varepsilon, \delta)$-PAC algorithms ensure the performance of the last-iterate policy, they share a common limitation: they halt learning once an $\varepsilon$-optimal policy is found.
This implies that existing $(\varepsilon, \delta)$-PAC CMDP algorithms never converge to optimal policies since they may make infinitely many mistakes with accuracy $\varepsilon / 2$ \citep{dann2017unifying}.

\looseness=-1
\paragraph{Convergence Guarantees of Primal-dual Algorithms without Exploration.}
Numerous studies have struggled to design primal-dual algorithms that provide convergence guarantees even in non-exploration settings.

\looseness=-1
One fundamental approach involves utilizing the average of model parameters.
By exploiting the convexity of CMDPs concerning the occupancy measure \citep{altman1999constrained}, a straightforward primal-dual method over the occupancy measure ensures that the average of the updated occupancy measures throughout training converges to an optimal solution \citep{zahavy2021reward}.
However, optimization over the occupancy measure becomes impractical when the state space is large.
To address this challenge, \citet{miryoosefi2019reinforcement}, \citet{chen2021primal}, \citet{li2021faster}, and \citet{liu2021policy} propose policy optimization algorithms.
While they offer certain performance guarantees for the average of past policies, there is no assurance for the last-iterate policy.
This poses a challenge in cases where policy averaging is impractical, as in deep RL applications.
\citet{ding2020natural} provide a policy gradient algorithm independent of policy averaging, but their performance is guaranteed only for the average of past performances, not for the last-iterate policy.

\looseness=-1
Rather than the average-based approach, certain studies try to ensure the performance of the last-iterate policy.
\citet{moskovitz2023reload} propose a primal-dual algorithm with a last-iterate convergence guarantee. However, their optimization is over the occupancy measure, and furthermore, they do not provide non-asymptotic performance guarantees.
\citet{ying2022dual} achieve a non-asymptotic performance guarantee by employing a policy-entropy regularized Lagrange function, and \citet{ding2023last} present an extended algorithm regularized by both policy-entropy and Lagrange multipliers.
However, their algorithms converge to a biased solution rather than the optimal solution due to the regularization.

\section{Other Performance Measures}\label{sec:other performance measure}
\looseness=-1
We use the following notations to define the high-probability regret.
\begin{equation}
\regret^K \df \sum^K_{k=1} \gapret^k, \; \regvio^K \df \sum^K_{k=1} \gapvio^k,\; \label{eq:regret-and-missnum}
\end{equation}
Intuitively, $\regret^K$ and $\regvio^K$ quantify the cumulative optimality gap and cumulative constraint violation up to episode $K \in \N$, respectively. 

\begin{definition}[Regret]\label{definition:regret}
\looseness=-1
For an episode length $K \in \N$ and a failure probability $\delta \in (0, 1]$, let $F_{\mathrm{HPR}} \paren{\cdots}$ be shorthand for $F_{\mathrm{HPR}}\paren*{X, A, H, K, N, 1 / \bgap, \ln\paren*{1 / \delta}}$, where $F_{\mathrm{HPR}}$ is a real-valued function polynomial in all its arguments.
An algorithm achieves $F_{\mathrm{HPR}}(\cdots)$ regret for $\delta$ if there exists $F_{\mathrm{HPR}}(\cdots)$ such that
\begin{equation*}
\begin{aligned}
\P\paren*{
\regret^K > F_{\mathrm{HPR}}(\cdots)
\;\vee\;
\regvio^K > F_{\mathrm{HPR}}(\cdots)
} \leq \delta\;,
\end{aligned}
\end{equation*}
where $\P$ is for the randomness of $\pi^k$ due to the stochastic interaction to the CMDP.
We say that the regret is sublinear if 
$F_{\mathrm{HPR}}(\cdots)$ is sublinear with respect to $K$.
\end{definition}

\begin{definition}[$(\varepsilon, \delta)$-PAC]\label{definition:PAC}
\looseness=-1
For an admissible accuracy $\varepsilon > 0$, let $F_{\mathrm{PAC}} \paren{\cdots}$ be shorthand for $F_{\mathrm{PAC}}\paren*{X, A, H, N, 1 / \bgap, 1 /\varepsilon, \ln\paren*{1 / \delta}}$, where $F_{\mathrm{PAC}}$ is a real-valued function polynomial in all its arguments.
For $\varepsilon > 0$ and $\delta > 0$, an algorithm achieves $(\varepsilon, \delta)$-PAC if there exists $F_{\mathrm{PAC}}(\cdots)$ such that 
\begin{equation*}
\P\paren*{
\missret^\varepsilon > 
F_{\mathrm{PAC}}(\cdots)
\;\vee\;
\missvio^\varepsilon > 
F_{\mathrm{PAC}}(\cdots)
} \leq \delta\;.
\end{equation*}
\end{definition}

\looseness=-1
\begin{remark}[Weak Regret Measures]\label{remark:performance measure}
Rather than $\regret^K$ and $\regvio^K$, the regret of an algorithm is often measured by the following $\regretm^K$ and $\regviom^K$ \citep{efroni2020exploration,liu2021learning,bura2022dope,wei2021provably}:
\begin{equation}\label{eq:cancel-regret}
\regretm^K \df \sum^K_{k=1} \gapretm^k\;, \; \regviom^K \df \sum^K_{k=1} \gapviom^k\;.
\end{equation}
Note that $\gapretm^k$ and $\gapviom^k$ might be negative since policy $\pi^k$ might violate the constraints. 
The negative temporal value allows an algorithm to cancel the past positive regret with the future negative regret. 
On the other hand, such \emph{error cancellations} are not permitted in $\regret^K$ and $\regvio^K$.
Hence, regret guarantees for $\regret^K$ and $\regvio^K$ are stronger than those for $\regretm^K$ and $\regviom^K$ in the sense that guarantees on the former imply guarantees on the latter, but not vice versa.
\end{remark}

\section{Naive Primal-Dual RL for Online CMDPs}\label{subsec:classical-primal-dual}

\begin{algorithm}[H]
    \caption{Naive Primal-Dual RL for online CMDPs}
    \begin{algorithmic}
    \STATE {\bfseries Input:} $\delta \in (0, 1]$ and iteration length $K \in \N$
    \STATE Set $\lambda^1 \df \bzero$. Set $\pi^1_h(a\mid x) \df \frac{1}{\aA} \quad \forall (x, a, h) \in \HXA$
    \FOR{$k = 1, \dots K$}
        \STATE Compute bonus ${\beta}^{k,\delta}$ by \cref{eq:classical-bonus}
        \STATE Compute $\widetilde{L}^k(\pi, \lambda)$ by \cref{eq:classical-optimistic-lagrange}
        \STATE Compute $\pi^{k+1}$ by a policy optimization over $\widetilde{L}^k(\cdot, \lambda^k)$ 
        \STATE Compute $\lambda^{k+1}$ by a gradient descent over $\widetilde{L}^k(\pi^{k}, \cdot)$ 
        \STATE Rollout $\pi^{k+1}$ and then update $n^k$ and $\barP^k$
    \ENDFOR
    \end{algorithmic}\label{algo:classical-primal-dual}
\end{algorithm}

\looseness=-1
For a better understanding of our proposed algorithm in \cref{sec:proposal}, this section introduces the naive PD-RL algorithm for online CMDPs under \cref{assumption:slater} (e.g., \citet{efroni2020exploration,liu2021learning,wei2021provably}).
We provide the pseudocode of the algorithm as \cref{algo:classical-primal-dual}.

\looseness=-1
Let $L: \Pi \times \R_+^N \to \R$ be the Lagrange function such that for a policy $\pi \in \Pi$ and its multipliers $\lambda \in \R_+^N$, 
\begin{equation}\label{eq:classical-lagrange}
L(\pi, \lambda) \df 
\vf{\pi}_{1}\brack*{P, r^0}(x_1)
+ \sum_{n=1}^N\lambda^n \paren*{\vf{\pi}_{1}\brack*{P, r^n}(x_1) - b^n}\;.
\end{equation}

\looseness=-1
Let $\lambda^\star \in \argmin_{\lambda \in \R_+^N} \max_{\pi \in \Pi} L(\pi, \lambda)$.
The central idea of the naive algorithm involves exploring the environment to identify a pair $(\pi^\star, \lambda^\star)$, which is a saddle point of $L$ \citep{altman1999constrained}. In other words, the pair satisfies:
$
L(\pi^\star, \lambda) \geq 
L(\pi^\star, \lambda^\star) \geq 
L(\pi, \lambda^\star)
$
for any $(\pi, \lambda) \in \Pi\times \R^N_+$.

\looseness=-1
The key to encouraging exploration is adhering to the \emph{optimism-in-the-face-of-uncertainty} principle, which propels the agent with an optimistic policy. 
Given a failure probability $\delta \in (0, 1]$ and a fixed iteration length $K \in \N$, the principle is typically realized by introducing a bonus term ${\beta}^{k,\delta}_\cdot (\cdot, \cdot): \HXA \to \R$ into the reward function, where $\beta^{k,\delta}$ takes the form of
\begin{equation}\label{eq:classical-bonus}
{\beta}^{k, \delta}_h(x, a) \approx
\frac{\ln \paren*{{K}{\delta^{-1}}}}{n_{h}^k(x, a) \vee 1}
+
\sqrt{\frac{\ln\paren*{{K}{\delta^{-1}}}}{n_{h}^k(x, a) \vee 1}
}\;.
\end{equation}
Here, we highlight the terms dependent on $n^k_h(x, a)$, $K$, and $\delta$ and conceal all other relevant terms for simplicity.
Using $\barP^k$ and $\beta^{k,\delta}$, the algorithm computes an optimistic Lagrange function 
$\widetilde{L}^k: \Pi \times \R^N_+ \to \R$ such that
\begin{equation}\label{eq:classical-optimistic-lagrange}
\widetilde{L}^k(\pi, \lambda) = 
\vf{\pi}_{1}\brack*{\barP^k, r^0 + \beta^{k,\delta}}(x_1)
+ \sum_{n=1}^N\lambda^n \paren*{\vf{\pi}_{1}\brack*{\barP^k, r^n + \beta^{k, \delta}}(x_1) - b^n}\;.
\end{equation}
Let $\pi^k \in \Pi$ and $\lambda^k \in \R^N_+$ be the policy and the Lagrange multipliers at the episode $k$, respectively.
The algorithm computes $\pi^{k+1}$ by a policy optimization method, such as policy gradient, with the aim of maximizing $\widetilde{L}^k(\cdot, \lambda^k)$.
Subsequently, it computes $\lambda^{k+1}$ by a projected gradient descent to minimize $\widetilde{L}^k(\pi^k, \cdot)$. 
The naive algorithm iterates this update scheme until the episode reaches $K$.

\looseness=-1
\paragraph{Challenges towards a Uniform-PAC Algorithm.}
While the naive algorithm attains sublinear regret \citep{efroni2020exploration,liu2021learning,wei2021provably}, it may fall short in delivering Uniform-PAC guarantees, encountering two primary challenges. 

\looseness=-1
Firstly, even in the absence of exploration, i.e., $\barP^k = P$ for all $k$, finding a saddle point $(\pi^\star, \lambda^\star)$ of the Lagrange function in \cref{eq:classical-lagrange} is non-trivial. 
Even if we have $\lambda^k=\lambda^\star$ and an associated maximum policy $\pi^{k+1} \in \argmax_{\pi \in \Pi}L(\pi, \lambda^\star)$, it is not guaranteed that $\pi^{k+1}$ represents an optimal policy.
In some CMDPs where feasible policies in $\Pisafe$ must be stochastic \citep{altman1999constrained}, the maximization can provide a deterministic $\pi^{k+1}$ that cannot be feasible.
Hence, ensuring that $(\pi^{k}, \lambda^{k})$ will be close to $(\pi^\star, \lambda^\star)$ poses a potential challenge.

\looseness=-1
Secondly, the naive bonus function in \cref{eq:classical-bonus} might not be adequate for achieving a uniform PAC guarantee. 
The inclusion of $\ln(K)$ in the bonus function implies that the algorithm attempts each action in each state infinitely often, potentially leading to an infinite number of mistakes \citep{dann2017unifying}.
\section{Experiment Details}\label{sec:experiment details}

\paragraph{Environment Construction.}
We instantiated a CMDP with $X=30$, $A=3$, $H=10$, and $N=1$, employing a construction strategy akin to that of \citet{dann2017unifying}.
For all $x, a, h$, the transition probabilities $P_h\paren*{\cdot \given x, a}$ were independently sampled from $\operatorname{Dirichlet}(0.1, \dots, 0.1)$.
This transition probability kernel is concentrated yet encompasses non-deterministic transition probabilities.

The reward values for the objective $r_h^0(x, a)$ are set to $0$ with probability $0.5$ and uniformly chosen at random from $[0, 1]$ otherwise.
The reward values for the constraint $r_h^1(x, a)$ are set to $1 - r_h^0(x, a)$.
Thus, the constraint and objective are in conflict in this CMDP.
This aligns with the CMDP construction strategy proposed by \citet{moskovitz2023reload} to generate a straightforward CMDP where a naive primal-dual algorithm might struggle to converge.

The initial state $x_0$ is randomly chosen from $\X$ and fixed during the training.
The constraint threshold is set as $b^1 = \frac{1}{2}\max_{\pi \in \Pi}\vf{\pi}_1[P, r^1](x_1)$.

\paragraph{Algorithm Implementations.}
We modify the \texttt{OptPrimalDual-CMDP} algorithm from \citet{efroni2020exploration} for the setting where the reward functions $\br$ are known. 

All the algorithms use $\delta = 0.1$.
For \texttt{OptPrimalDual-CMDP} and \texttt{No-UPAC-bonus} that use the naive bonus function (\cref{eq:classical-bonus}), we set $K = 10^5$.
For \proposal and \texttt{No-UPAC-bonus}, we set $\alpha_\eta=0.53$ and $\alpha_\tau=0.4$, that do not contradict to \cref{theorem:algorithm-is-uniform-PAC}.
For \texttt{No-regularization}, we set $\alpha_\eta=0.53$ and $\tau_k = 0$.
For \texttt{No-adjustment}, we set $\tau_k = \eta_k = 0.1$.

Finally, we scale the bonus functions of all the algorithms by a factor of $10^{-3}$ to observe the algorithms' behavior in relatively smaller episodes.
\section{Notation for Proofs}\label{appendix:notation}

For any $h \in [H]$, let $w^{\pi}_\cdot[P]: [H] \to \Delta(\XA)$ be the occupancy measure of $\pi$ in $P$.
In other words, for any $(h, x, a) \in \HXA$, it satisfies
\begin{equation}\label{eq:occupancy measure}
w^\pi_h[P](x, a) = \P(x_h=x, a_h=a\mid \pi, P)\;.
\end{equation}
With an abuse of notation, we write $w^\pi_h[P](x) = \sum_{a \in \A} w^\pi_h[P](x, a)$.

Let $\tvf{k}_\cdot(\cdot): [H] \times \X\to \R$ be the regularized optimistic value function such that: 
\begin{equation}\label{eq:tvf-definition}
\tvf{k}_h(x) = \tvf{k,0}_h(x) +  \sum_{n=1}^N\lambda^{k,n}\tvf{k,n}_h(x)\quad \forall (h, x) \in [H]\times\X\;,
\end{equation}
where $\tvf{k;n}$ for $n \in \brace{0}\cup [N]$ are defined in \cref{algo:primal-dual}.

We use the shorthand $\Hent \df H(1+\ln \aA)$.
Finally, for a set of positive values $\brace*{a_n}_{n=1}^N$, we write $x=\polylog\left(a_1, \dots, a_N\right)$ if there exists an absolute constant $C>0$ and $\alpha > 0$ such that $x \leq C \paren*{\sum_{n=1}^N\ln a_n}^\alpha$ holds.

\section{Useful Lemmas}\label{appendix:useful-lemma}

\subsection{RL Useful Lemmas}

\begin{lemma}[\textbf{Lemma 34} in \citet{efroni2020exploration}]\label{lemma:extended value difference}
Let $\pi, \pi'$ be two policies, $P$ be a transition model, and $g$ be a reward function.
Let $\tvf{\pi}_\cdot(\cdot): [H] \times \X \to \R$ be a function such that $\tvf{\pi}_h = \pi_h \widetilde{Q}_h$ for all $h \in [H]$, where we defined $\widetilde{Q}_\cdot(\cdot, \cdot): \HXA \to \R$.
Then, 

\begin{equation*}
\begin{aligned}
& \tvf{\pi}_1(x_1)-\vf{\pi'}_1\brack*{P;g}(x_1) \\
=& \sum_{h=1}^H \sum_{x \in \X} w_h^{\pi'}[P](x)\sum_{a \in \A}\paren*{\pi_h^{\prime}\paren*{a\mid x}-\pi_h\paren*{a \mid x}}\widetilde{Q}_h(x, a) \\
& +\sum_{h=1}^H \sum_{x, a \in \XA} w_h^{\pi'}[P](x, a)\paren*{\widetilde{Q}_h\left(x, a \right)-g_h\left(x, a\right)-\paren*{P_h'\tvf{\pi}_{h+1}}(x, a)}
\end{aligned}
\end{equation*}
\end{lemma}

\begin{lemma}[Error to regret]\label{lemma:error-to-regret}
Consider two sequences of real values $x_1, x_2, \dots$ and $y_1, y_2, \dots$. 
Assume that $0 \leq x_i, y_i \leq u$ for any $i \in \N$ with $u > 0$.
For any $\varepsilon \in (0, u]$, assume that 
\begin{equation*}
x_k - y_k \leq \varepsilon
\end{equation*}
on all $k \in \N$ except at most
\begin{equation*}
\left\lceil\frac{Z_1}{\varepsilon^{\alpha}} \paren*{\ln \paren*{\frac{Z_2 }{\varepsilon}}}^\beta\right\rceil
\end{equation*}
times, 
where $\alpha \geq 1$, $\beta \geq 0$, $Z_1 > 0$, and $Z_2 \geq \max\brace{u, e}$ are constants that do not depend on $\varepsilon$.
We also assume that $Z_2 \geq e Z_1^{\frac{1}{\alpha}}$.
Then, for any $K \in \N$, it holds that
\begin{align*}
    \sum^K_{k=1} \max\brace*{x_k - y_k, 0}
    \leq
    K^{1-\frac{1}{\alpha}} Z_1^{\frac{1}{\alpha}} \polylog(K, Z_1, Z_2, u)\;.
\end{align*}
\end{lemma}
\begin{proof}
This proof follows the strategy of \textbf{Theorem A.1} in \citet{dann2017unifying}.

Let $z\df\paren*{\frac{K}{Z_1}}^{\frac{1}{\alpha}}$ with $K \in \N$. 
Due to the assumption that 
$Z_2 \geq e Z_1^{\frac{1}{\alpha}}$,
we have
\begin{equation}\label{eq:tmp C2_z}
Z_2 z = \frac{Z_2}{Z_1^{\frac{1}{\alpha}}} K^{\frac{1}{\alpha}} 
\geq
\frac{Z_2}{Z_1^{\frac{1}{\alpha}}} 
\geq e\;.
\end{equation}
Also, let 
$g(\varepsilon)\df
\frac{Z_1}{\varepsilon^{\alpha}} \paren*{\ln \paren*{\frac{Z_2 }{\varepsilon}}}^\beta
$ and $\epsmin \df \frac{\paren*{\ln(Z_2 z)}^{\frac{\beta}{\alpha}}}{z}$.
Note that $g(\varepsilon)$ is well-defined since $\ln {Z_2} - \ln {\varepsilon} \geq 0$ due to $Z_2 \geq \max\brace{u, e}$.
Then, it holds that
\begin{align*}
g(\epsmin) &= \frac{Z_1}{\epsmin^\alpha}\paren*{\ln \frac{Z_2}{\epsmin}}^\beta
=
Z_1\frac{z^\alpha}{\paren*{\ln\paren*{Z_2 z}}^\beta}\paren*{\ln \paren*{Z_2 z} - \frac{\beta}{\alpha}
\underbrace{\ln \ln (Z_2 z)}_{\geq 0 \text{ due to \cref{eq:tmp C2_z}}}}^\beta\\   
&\leq 
Z_1\frac{z^\alpha}{\paren*{\ln\paren*{Z_2 z}}^\beta}\paren*{\ln \paren*{Z_2 z}}^\beta
= Z_1 z^\alpha
= K\;.
\end{align*}
Since $g(\varepsilon)$ is monotonically decreasing for $\varepsilon > 0$, due to $ g(\epsmin) \leq K$, we have $g(\varepsilon) \leq K$ for any $\varepsilon \in [\epsmin, u]$.
Using the above results with $x_k - y_k \leq u$ for any $k$, it holds that
\begin{equation*}
    \sum^K_{k=1} \max\brace*{x_k - y_k, 0}
    \leq
    \int^u_0 g(\varepsilon)
    \leq 
    K \epsmin + \int^u_{\epsmin} g(\varepsilon) d\varepsilon \;.
\end{equation*}
For the intuition of the above inequality, see \textbf{Figure 3} in \citet{dann2017unifying}.
We are going to bound both terms in the right-hand side separately.

For the first term, we have
\begin{equation*}
K \epsmin = 
K \frac{\paren*{\ln Z_2 + \frac{1}{\alpha}\ln K - \frac{1}{\alpha}\ln Z_1}^{\frac{\beta}{\alpha}}}{\paren*{\frac{K}{Z_1}}^{\frac{1}{\alpha}}}
= 
K^{1-\frac{1}{\alpha}} Z_1^{\frac{1}{\alpha}}\polylog(K, Z_1, Z_2)\;.
\end{equation*}

For the second term, we have
\begin{equation*}
\int^u_{\epsmin} g(\varepsilon) d\varepsilon   
= 
\int^u_{\epsmin} 
\frac{Z_1}{\varepsilon^{\alpha}} \paren*{\ln \paren*{\frac{Z_2 }{\varepsilon}}}^\beta
d\varepsilon   
\leq 
Z_1 \paren*{\ln \paren*{\frac{Z_2 }{\varepsilon}}}^\beta
\int^u_{\epsmin} 
{\varepsilon^{-\alpha}} 
d\varepsilon \;.
\end{equation*}
When $\alpha=1$, we have
$$
Z_1 \paren*{\ln \paren*{\frac{Z_2 }{\varepsilon}}}^\beta
\int^u_{\epsmin} 
{\varepsilon^{-1}} 
d\varepsilon 
= 
Z_1 \paren*{\ln \paren*{\frac{Z_2 }{\varepsilon}}}^\beta
\paren*{\ln u - \ln \epsmin}
= Z_1 \polylog\paren*{K, Z_1, Z_2, u}\;.
$$
When $\alpha> 1$, we have
$$
Z_1 \paren*{\ln \paren*{\frac{Z_2 }{\varepsilon}}}^\beta
\int^u_{\epsmin} 
{\varepsilon^{-\alpha}} 
d\varepsilon
\leq
\frac{Z_1}{\alpha - 1}
\paren*{\ln \paren*{\frac{Z_2 }{\varepsilon}}}^\beta
\epsmin^{1 - \alpha}\
= 
K^{1-\frac{1}{\alpha}}
Z_1^{\frac{1}{\alpha}}
\polylog(K, Z_1, Z_2)\;,
$$
where we used
$
\epsmin^{1-\alpha} 
= K^{1-\frac{1}{\alpha}}
Z_1^{\frac{1}{\alpha} -1}
\polylog(K, Z_1, Z_2)
$.

Therefore, we conclude that
\begin{equation*}
    \sum^K_{k=1} \max\brace*{ x_k - y_k , 0} 
    \leq
    K^{1-\frac{1}{\alpha}}
    Z_1^{\frac{1}{\alpha}}
    \polylog(K, Z_1, Z_2, u)\;.
\end{equation*}
\end{proof}

\begin{lemma}\label{lemma:reg-suboptimality}
Suppose that \cref{assumption:slater} holds.
For any $\tau \in (0, \infty)$ and $(\pi, \lambda) \in \Pi\times \R^N_+$, it holds that
\begin{equation*}
\begin{aligned}
&\vf{\pi}_{1}[P, r^0](x_1)
+ \sum^N_{n=1} \lambda^{\star,n}_\tau \paren*{\vf{\pi}_1[P, r^n](x_1) - b^n}
\leq 
\vf{\pi_{\tau}^\star}_{1}[P, r^0](x_1)
+ \sum^N_{n=1} \lambda^{\star,n}_\tau \paren*{\vf{\pi_\tau^\star}_1[P, r^n](x_1) - b^n}
+ \tau \Hent\;, \\
&
\text{ and }\;
\sum^N_{n=1} \lambda^{\star,n}_\tau \paren*{\vf{\pi_\tau^\star}_1[P, r^n](x_1) - b^n}
\leq 
\sum^N_{n=1} \lambda^n \paren*{\vf{\pi_\tau^\star}_1[P, r^n](x_1) - b^n}
+ \frac{\tau}{2}\norm*{\lambda}^2_2\;.
\end{aligned}
\end{equation*}
\end{lemma}
\begin{proof}
$
L_\tau(\pi, \lambda^\star_\tau)
\leq L_\tau(\pi^\star_\tau, \lambda^\star_\tau)
$ due to \cref{lemma:strong duality} indicates that
\begin{equation*}
\begin{aligned}
&\vf{\pi}_{1}\brack*{P, r^0}(x_1)
+ \sum_{n=1}^N\lambda^{\star,n}_\tau \paren*{\vf{\pi}_{1}\brack*{P, r^n}(x_1) - b^n}
+ \frac{\tau}{2} \|\lambda^{\star}_\tau\|_2^2\\
\leq 
&\vf{\pi}_{1}\brack*{P, r^0, \tau}(x_1)
+ \sum_{n=1}^N\lambda^{\star,n}_\tau \paren*{\vf{\pi}_{1}\brack*{P, r^n}(x_1) - b^n}
+ \frac{\tau}{2} \|\lambda^{\star}_\tau\|_2^2\\
\leq 
&\vf{\pi^\star_\tau}_{1}\brack*{P, r^0}(x_1) + \tau\Hent
+ \sum_{n=1}^N\lambda^{\star,n}_\tau \paren*{\vf{\pi^\star_\tau}_{1}\brack*{P, r^n}(x_1) - b^n}
+ \frac{\tau}{2} \|\lambda_\tau^\star\|_2^2\;.
\end{aligned}
\end{equation*}
The first claim holds by rearranging the above inequality.

Also, 
$
L_\tau(\pi^\star_\tau, \lambda^\star_\tau)
\leq L_\tau(\pi^\star_\tau, \lambda)
$ due to \cref{lemma:strong duality} indicates that
\begin{equation*}
\begin{aligned}
&\vf{\pi^\star_\tau}_{1}\brack*{P, r^0, \tau}(x_1)
+ \sum_{n=1}^N\lambda^{\star,n}_\tau \paren*{\vf{\pi^\star_\tau}_{1}\brack*{P, r^n}(x_1) - b^n}\\
\leq&\vf{\pi^\star_\tau}_{1}\brack*{P, r^0, \tau}(x_1)
+ \sum_{n=1}^N\lambda^{\star,n}_\tau \paren*{\vf{\pi^\star_\tau}_{1}\brack*{P, r^n}(x_1) - b^n}
+ \frac{\tau}{2} \|\lambda^{\star}_\tau\|_2^2\\
\leq 
&\vf{\pi^\star_\tau}_{1}\brack*{P, r^0, \tau}(x_1)
+ \sum_{n=1}^N\lambda^n \paren*{\vf{\pi^\star_\tau}_{1}\brack*{P, r^n}(x_1) - b^n}
+ \frac{\tau}{2} \|\lambda\|_2^2\;.
\end{aligned}
\end{equation*}
The second claim holds by rearranging the above inequality.
\end{proof}

\begin{lemma}[Properties of $\lambda^\star_\tau$ and $\pi^\star_\tau$]\label{lemma:property-of-regularized}
Suppose that $0 \leq \tau \leq 1$.
Under \cref{assumption:slater}, we have    
$$
\sum^N_{n=1} \lambda^{\star,n}_\tau \leq \frac{\Hent}{\bgap}
\;\text{ and }\;
\vf{\pi^\star_\tau}_1[P, r^n](x_1) - b^n \geq - \frac{\tau\Hent}{\bgap} \quad \forall n \in [N]\;.
$$
\end{lemma}
\begin{proof}


We first assume $\tau > 0$.
Since $\lambda^{\star,n}_\tau = \argmin_{\lambda^n\in \R_+} \lambda^n \paren*{\vf{\pi^\star_\tau}_{1}[P, r^n](x_1) - b^n} + \frac{\tau}{2}\paren*{\lambda^n}^2$, $\lambda^{\star,n}_\tau$ can be analytically computed as
$
    \dfrac{1}{\tau} \max \brace*{
        b^n - \vf{\pi^\star_\tau}_{1}[P, r^n](x_1), 0
    }
$ for every $n \in [N]$. Thus, from \cref{lemma:strong duality} and \cref{assumption:slater},
\begin{equation}\label{eq:lambda-range tmp}
\begin{aligned}
    L_\tau(\pisafe, \lambda^\star_\tau)
    &=
        \vf{\pisafe}_{1}[P, r^0, \tau](x_1)
        +
        \sum^N_{n=1} \lambda^{\star,n}_\tau \underbrace{
            \paren*{
                \vf{\pisafe}_1[P, r^n](x_1) - b^n
            }
        }_{\geq \bgap \text{ by \cref{assumption:slater}}}
    \\
    &\leq 
        \vf{\pi^\star_\tau}_{1}[P, r^0, \tau](x_1)
        + \underbrace{
            \sum^N_{n=1} \lambda^{\star,n}_\tau \paren*{
                \vf{\pi^\star_\tau}_1[P, r^n](x_1) - b^n
            }
        }_{\leq 0 \text{ from the analytical expression of }\lambda^{\star,n}_\tau}
    \leq
        \vf{\pi^\star_\tau}_{1}[P, r^0, \tau](x_1)\;.
\end{aligned}
\end{equation}
By rearrangement,
\begin{equation*}
    \bgap\sum^N_{n=1} \lambda^{\star,n}_\tau 
    \leq
        \vf{\pi^\star_\tau}_{1}[P, r^0, \tau](x_1)
        -  
        \vf{\pisafe}_{1}[P, r^0, \tau](x_1)
    \leq
        H(1 + \tau \ln \aA)
    \leq
        \Hent\;,
\end{equation*}
which concludes the proof of the first claim for $\tau > 0$.
Furthermore, as $\lambda^{\star,n}_\tau \leq \sum^N_{m=1} \lambda^{\star;m}_\tau$ for any $n \in [N]$,
\begin{equation*}
    b^n - \vf{\pi^\star_\tau}_{1}[P, r^n](x_1)
    \leq
        \max \brace*{
            b^n - \vf{\pi^\star_\tau}_{1}[P, r^n](x_1), 0
        }
    =
        \tau \lambda^{\star,n}_\tau 
    \leq
        \frac{\tau \Hent}{\bgap}\;,
\end{equation*}
which concludes the proof of the second claim for $\tau > 0$.

The first and second claims when $\tau = 0$ obviously hold.
Indeed,  
$\vf{\pi^\star_0}_1[P, r^n](x_1) - b^n \geq 0$ for every $n \in [N]$ by definition, and thus, $\lambda^{\star,n}_\tau \paren*{\vf{\pi^\star_\tau}_1[P, r^n](x_1) - b^n} = 0$ when $\tau = 0$, which provides \cref{eq:lambda-range tmp}.
\end{proof}

\begin{lemma}[KL to optimality gap]\label{lemma:KL-to-optimality}
Let $\tau > 0$ and $\pi \in \Pi$ with $\pi > \bzero$. 
Assume that \cref{assumption:slater} holds and
$$\sum_{h=1}^H \sum_{x \in \X} w_h^{\pi^{\star}_{\tau}}[P](x)
\KL{\pi^\star_{\tau,h}\paren*{\cdot \given x}}{\pi_h\paren*{\cdot \given x}} \leq \varepsilon\;.$$
Then, $\pi$ satisfies
\begin{equation*}
\begin{aligned}
&V_1^{\pi^{\star}}[P, r](x_1)-
V_1^{\pi}[P, r](x_1)
\leq \tau\Hent + H\sqrt{2H\varepsilon}\\
\text{ and }\;&
b^n -
V_1^{\pi}[P, r^n](x_1)
\leq \frac{\tau\Hent}{\bgap} + H\sqrt{2H\varepsilon} \quad \forall n \in [N]\;.
\end{aligned}
\end{equation*}
\end{lemma}
\begin{proof}
Note that
$$
V_1^{\pi^{\star}}[P, r](x_1)-
V_1^{\pi}[P, r](x_1)
=\underbrace{V_1^{\pi^{\star}}[P, r](x_1)-V_1^{\pi_{\tau}^{\star}}[P, r](x_1)}_{\mathrm{(i)}}+\underbrace{V_1^{\pi_{\tau}^{\star}}[P, r](x_1)-V_1^{\pi}[P, r](x_1)}_{\mathrm{(ii)}} .
$$

For the term $\mathrm{(ii)}$, we have
\begin{equation*}
\begin{aligned}
\mathrm{ (ii) } 
&= V_1^{\pi_{\tau}^{\star}}[P, r](x_1)-V_1^{\pi}[P, r](x_1)\\
&\numeq{\leq}{a}H \sum_{h=1}^H  \sum_{x \in \X} w_h^{\pi^{\star}_{\tau}}[P](x)\norm*{\pi_{\tau,h}^{\star}\paren*{\cdot\mid x}-\pi_h\paren*{\cdot \mid x}}_1\\
&\numeq{\leq}{b}H \sum_{h=1}^H \sqrt{\sum_{x \in \X} w_h^{\pi^{\star}_{\tau}}[P](x)\norm*{\pi_{\tau,h}^{\star}\paren*{\cdot\mid x}-\pi_h\paren*{\cdot \mid x}}_1^2}\\
&\numeq{\leq}{c}H \sum_{h=1}^H \sqrt{2\sum_{x \in \X} w_h^{\pi^{\star}_{\tau}}[P](x)
\KL{\pi^\star_{\tau,h}\paren*{\cdot \given x}}{\pi_h\paren*{\cdot \given x}}}\\
&\numeq{\leq}{d}H \sqrt{2H\sum_{h=1}^H \sum_{x \in \X} w_h^{\pi^{\star}_{\tau}}[P](x)
\KL{\pi^\star_{\tau,h}\paren*{\cdot \given x}}{\pi_h\paren*{\cdot \given x}}}
\leq H\sqrt{2H\varepsilon}\;,
\end{aligned}
\end{equation*}
where (a) is due to \cref{lemma:extended value difference}
(b) uses the fact that $(\E[x])^2 \leq \E[x^2]$,
(c) uses the Pinsker's inequality,
and (d) uses \cref{lemma:sum of sqrt}.

For the term $\mathrm{(i)}$, the second claim of \cref{lemma:reg-suboptimality} with $\lambda=\bzero$ indicates that
$
\sum^N_{n=1} \lambda^{\star,n}_\tau \paren*{\vf{\pi_\tau^\star}_1[P, r^n](x_1) - b^n}
\leq 
0\;.
$
Then, the first claim of \cref{lemma:reg-suboptimality} with $\pi=\pi^\star$ indicates that
\begin{equation*}
\mathrm{(i)} = \vf{\pi^\star}_{1}[P, r^0](x_1)
- \vf{\pi_{\tau}^\star}_{1}[P, r^0](x_1)
\leq \tau \Hent\;.
\end{equation*}
Therefore, the optimality gap is bounded as
$$
V_1^{\pi^{\star}}[P, r^0](x_1)-
V_1^{\pi^k}[P, r^0](x_1)
\leq \tau\Hent + H\sqrt{2H\varepsilon}\;.
$$
This concludes the proof of the first claim.

For the second claim, for any $n \in [N]$, we have
$$
b^n -
V_1^{\pi^k}[P, r^n](x_1)
=\underbrace{b^n-V_1^{\pi_{\tau_k}^{\star}}[P, r^n](x_1)}_{\mathrm{(i)}}+\underbrace{V_1^{\pi_{\tau_k}^{\star}}[P, r^n](x_1)-V_1^{\pi^k}[P, r^n](x_1)}_{\mathrm{(ii)}} .
$$

By taking a similar transformation of the first claim, the term $\mathrm{(ii)}$ can be bounded by 
$
\mathrm{(ii)} \leq H\sqrt{2H \varepsilon }
$.
Furthermore, 
\cref{lemma:property-of-regularized} indicates that
$
\mathrm{(i)} \leq \frac{\tau\Hent}{\bgap}
$.
Therefore, we have
$$
b^n -
V_1^{\pi^k}[P, r^n](x_1)
\leq \frac{\tau\Hent}{\bgap} + H\sqrt{2H\varepsilon} \quad \forall n \in [N]\;.
$$
\end{proof}



\subsection{Other Useful Lemmas}

\begin{lemma}\label{lemma:square-log inequality}
For $x \in \Delta(\A)$ with $\abs*{\A}=A$ and $0 < c \leq 1$, it holds that
$$\sum_{a\in \A} \paren*{x(a)}^{c} \paren*{\ln x(a)}^2 \leq 
A^{1-c}\paren*{\frac{2}{c} - 1 + \ln A}^2\;.$$
\end{lemma}
\begin{proof}
Note that the second derivative of $f(x) \df x^c \paren*{1 - \frac{2}{c} + \ln x}^2$ is 
$$
\frac{\partial^2}{\partial x^2}f(x)
=x^{c-2}\paren*{\ln (x)\paren*{\underbrace{(c-1) c \ln (x)}_{\geq 0}+\underbrace{2(c-1) c+2}_{\geq 0}}+(c-1) c} \leq 0
$$
Therefore, $f(x)$ is a concave function.

Accordingly, 
\begin{equation*}
\begin{aligned}
\sum_{a\in\A} \paren*{x(a)}^c \paren*{\ln x(a)}^2
\leq
\sum_{a\in\A} \paren*{x(a)}^c \paren*{1 -\frac{2}{c} + \ln x(a)}^2
\leq
A^{1-c} \paren*{\frac{2}{c} - 1 + \ln A}^2\;,
\end{aligned}
\end{equation*}
where the second inequality is due to Jensen's inequality.

\end{proof}

\begin{lemma}[\textbf{Lemma 12} in \citet{cai2023uncoupled}]\label{lemma:x-1-y-x}
For $x \in (0, 1)$ and $y > 0$, we have $x^{1-y}-x \leq-y x^{1-y} \ln x$.
\end{lemma}

\begin{lemma}[Inequality for Mirror Descent]\label{lemma:online mirror ascent}
For $\ell: \A \to \R$, $x \in \Delta(\A)$, $1 \geq \eta > 0$, and $1 \geq \kappa \geq 0$, let 
$$
x^{\prime}=\underset{\tilde{x} \in \Delta(\A)}{\argmin}\left\{\sum_{a \in \mathcal{A}} \tilde{x}(a)\paren*{\ell(a) + \kappa \ln x(a)}+\frac{1}{\eta} \KL{\tilde{x}}{x}\right\}
$$
Then, for any $u \in \Delta(\A)$, it holds that
$$
\sum_{a \in \A}(x(a)-u(a))\paren*{\ell(a) + \kappa \ln x(a)} \leq 
\frac{\KL{u}{x}-\KL{u}{x^{\prime}}}{\eta}
+ 2\eta \kappa^2 \paren*{1 + \ln \aA}^2
+2\eta \sum_a x(a) \paren*{\ell(a)}^2\;.
$$
\end{lemma}
\begin{proof}
By the standard analysis of online mirror descent (e.g., \textbf{Lemma 14} from \citet{chen2021impossible}), we have
\begin{equation}\label{eq:three point lemma}
\begin{aligned}
\sum_{a \in \A}(x_a-u_a)\paren*{\ell_a + \tau \ln x_a} \leq 
\frac{\KL{u}{x}-\KL{u}{x^{\prime}}}{\eta}
+ \underbrace{\sum_{a \in \A}(x_a-x'_a)(\ell_a + \tau \ln x_a) - \frac{1}{\eta}\KL{x'}{x}}_{\clubsuit}\;.
\end{aligned}
\end{equation}

\looseness=-1
We will bound $\clubsuit$.
For two positive vectors $x, y \in \R^\aA$ such that $x(a), y(a) > 0$, define a mapping $\phi$ such that, 
$$
\phi(x, y) = \sum_{a} x(a) \paren*{\ln x(a) - \ln y(a)} - x(a) + y(a)\;.
$$
Note that when $x, y \in \Delta(\A)$, $\phi(x, y)$ is equivalent to $\KL{x}{y}$.
Then, for any $g \in \R^\aA$ and $x' \in \Delta(\A)$, 
\begin{equation*}
\begin{aligned}
&\sum_{a \in \A}-x'(a)g(a) - \frac{1}{\eta}\KL{x'}{x}\\
=&\sum_{a \in \A}-x'(a)g(a) - \frac{1}{\eta}\phi(x',x)\\
\leq &\max_{y \in \R^A}\sum_{a \in \A}-y(a)g(a) - \frac{1}{\eta}\phi(y, x)\\
= &\max_{y \in \R^A}\sum_{a \in \A}-y(a)g(a) - \frac{1}{\eta}
\sum_{a} y(a) \paren*{\ln y(a) - \ln x(a) - 1} + x(a)\\
= &\max_{y \in \R^A} f(y) - \frac{1}{\eta}\sum_{a \in \A} x(a)\;,
\end{aligned}    
\end{equation*}
where we defined a function
$f: y \in \R^A \mapsto \sum_{a \in \A}-y(a)g(a) - \frac{1}{\eta}
\sum_{a} y(a) \paren*{\ln y(a) - \ln x(a) - 1}$.
Note that $f(y)$ is a strongly concave function and has a unique maxima. Let $y^\star \df \argmax_{y \in \R^A} f(y)$.
It is easy to verify that $y^\star$ satisfies
$$
y^\star(a) = x(a)\exp(-\eta g(a)) \;.
$$
As $\ln y^\star (a) = \ln x (a) - \eta g(a)$, we have
$\sum_{a \in \A} -y^\star(a) g(a) 
= \frac{1}{\eta}\sum_{a \in \A} y^\star(a) \paren*{\ln y^\star(a) - \ln x(a)}
$ and $\sum_{a \in \A} x(a) g(a)  =  \frac{1}{\eta}\sum_{a \in \A} x(a) \paren*{\ln y^\star(a) - \ln x(a)}$.
Therefore, 
\begin{equation}\label{eq:mirror descent bound temp}
\begin{aligned}
\sum_{a \in \A}(x(a)-y^\star(a)) g(a) - \frac{1}{\eta}\phi(y^\star, x)
&= 
\frac{1}{\eta} \phi(x, y^\star)\\
&= 
\frac{1}{\eta}\sum_{a \in \A}
x(a) \paren*{\ln x(a) - \ln y^\star(a)} - x(a) + y^\star(a) \\
&=\frac{1}{\eta}\sum_{a \in \A}
x(a) \paren*{\eta g(a) - 1 + \exp(-\eta g(a))}\\
&\numeq{\leq}{a}
\frac{1}{\eta}\sum_{a \in \A}
x(a) \paren*{\eta g(a)}^2
= \eta \sum_{a \in \A}
x(a) \paren*{g(a)}^2\;,
\end{aligned}
\end{equation}
where (a) uses $\exp(-z) - 1 + z \leq z^2$ for $z\geq -1$.


Then, by setting $g(a) = \ell(a) + \tau \ln x(a)$ in \cref{eq:mirror descent bound temp}, $\clubsuit$ can be bounded as
\begin{equation*}
\begin{aligned}
\clubsuit = 
&\sum_{a \in \A}(x(a)-x'(a))(\ell(a) + \tau \ln x(a)) - \frac{1}{\eta}\KL{x'}{x}\\
= &\sum_{a \in \A}(x(a)-x'(a))(\ell(a) + \tau \ln x(a)) - \frac{1}{\eta}\phi(x',x)\\
\leq &\max_{y \in \R^A}\sum_{a \in \A}(x(a)-y(a))(\ell(a) + \tau \ln x(a)) - \frac{1}{\eta}\phi(y, x)\\
\leq & \eta \sum_{a \in \A} x(a) \paren*{\ell(a) + \tau \ln x(a)}^2\\
\numeq{\leq}{a} & 
2\eta \sum_{a \in \A} x(a) \paren*{\ell(a)}^2 + 
2\eta \tau^2\sum_{a \in \A} x(a) \paren*{\ln x(a)}^2\\
\numeq{\leq}{b} & 
2\eta \sum_{a \in \A} x(a) \paren*{\ell(a)}^2 + 
2\eta \tau^2\paren*{1 + \ln \aA}^2\\
\end{aligned}    
\end{equation*}
where (a) uses $(a+b)^2 \leq 2a^2 + 2b^2$ and (b) uses \cref{lemma:square-log inequality}.

The claim holds by plugging this result to \cref{eq:three point lemma}.
\end{proof}

\begin{lemma}[Inequality for Gradient Descent]\label{lemma:gradient descent}
For some $u, \lambda \in \R_+$, $\eta > 0$, and $\ell \in \R$, let
$\lambda' \df \clip\brack*{\lambda + \eta\ell, 0, u}\;.$
Then, for any $\lambda^\star \in [0, u]$, 
\begin{equation*}
\begin{aligned}
\ell(\lambda - \lambda^\star)
\leq 
\frac{1}{2\eta}\paren*{(\lambda - \lambda^\star)^2 - (\lambda' - \lambda^\star)^2}
+ \frac{\eta}{2}\ell^2\;.
\end{aligned}    
\end{equation*}
\end{lemma}
\begin{proof}
Let $\bar{\lambda}' \df \lambda + \eta\ell$.
Since $\lambda^\star \in [0, u]$, we have
$$
({\lambda}' - \lambda^\star)^2
= (\clip\brack*{\widebar{\lambda}', 0, u} - \lambda^\star)^2
\leq 
(\widebar{\lambda}' - \lambda^\star)^2\;.
$$
Therefore,
\begin{equation*}
\begin{aligned}
({\lambda}' - \lambda^\star)^2
\leq
(\bar{\lambda}' - \lambda^\star)^2
= 
\paren*{
\lambda + \eta\ell
- \lambda^\star
}^2
=(\lambda - \lambda^\star)^2 - 2\eta \ell(\lambda - \lambda^\star) + \eta^2 \ell^2\;.
\end{aligned}
\end{equation*}
The claim holds by rearranging the above inequality.
\end{proof}

\begin{lemma}\label{lemma:sum of sqrt}
For any positive real numbers $x_1, x_2, \dots, x_n$, 
$\sum_{i=1}^n \sqrt{x_i} \leq \sqrt{n}\sqrt{\sum^n_{i=1}x_i}$.
\end{lemma}
\begin{proof}
Due to the Cauchy-Schwarz inequality, we have
$
\left(\frac{\sum_{i=1}^n \sqrt{x_i}}{n}\right)^2 \leq \frac{\sum_{i=1}^n x_i}{n}
$.
Taking the square root of the inequality proves the claim.
\end{proof}

\begin{lemma}\label{lemma:k-1 vs polylog k}
Let $g:\N \to \R$ be a function such that
\begin{equation*}
g(k) = Z_1 k^{-\alpha} \paren*{\ln(Z_2 k)}^\beta
\end{equation*}
where $\alpha, \beta > 0$, $Z_1 > 0$, and $Z_2 \geq 1$ are constants that do not depend on $k$.
Then, for any $\varepsilon \in (0, \infty)$, there exists a constant $k^\star= \tiO\paren*{Z_1^{\frac{1}{\alpha}}\varepsilon^{-\frac{1}{\alpha}}}$
such that $g(k) \leq \varepsilon$ for all $k \geq k^\star$.
\end{lemma}
\begin{proof}
Consider a function $\kappa \in [1/Z_2, \infty) \mapsto Z_1 \kappa^{-\alpha} \paren*{\ln(Z_2 \kappa)}^\beta$. Note that
\begin{align*}
    Z_1 \kappa^{-\alpha} \paren*{\ln(Z_2 \kappa)}^\beta \leq \varepsilon
    \iff
    \frac{1}{(Z_2 \kappa)^{\alpha / \beta}} \ln(Z_2 \kappa) \leq \paren*{ \frac{\varepsilon}{Z_1} }^{1/\beta} \frac{1}{Z_2^{\alpha / \beta}}
    \iff
    \frac{\ln x}{x^\eta} \leq c\;,
\end{align*}
where $x \df Z_2 \kappa$, $\eta \df \alpha / \beta$, and $c \df \varepsilon^{1/\beta} Z_1^{-1/\beta} Z_2^{ - \alpha / \beta} $.
Let $f(x) \df x^{-\eta} \ln x$. Its derivative is given by
\begin{align*}
    f' (x)
    =
    x^{-1-\eta} \paren*{ 1 - \eta \ln x }\;.
\end{align*}
Therefore $f$ is increasing when $1 \leq x < e^{1/\eta}$, takes its maximum $1/(\eta e)$ at $x = e^{1/\eta}$, and decreasing towards $0$.
Hence there exists some $x^\star$ such that $f(x) \leq c$ for all $x \geq x^\star$. Desired $k^\star$ can be obtained by $\ceil{x^\star / Z_2}$, where $\ceil{\cdot}$ is the ceiling function.

As $f(x) \leq 1/(\eta e)$, we assume $1/(\eta e) > c$ otherwise the claim trivially holds.
Following the same discussion as above, it can be shown that
a function $x \mapsto x^{-\eta \lambda} \ln x$ for $\lambda \in (0, 1)$ takes its maximum $1/(e \eta \lambda)$ at $x = e^{1/(\eta \lambda)}$.
Then, since
\begin{align*}
    f(x) = \frac{1}{x^{(1-\lambda) \eta}} \frac{1}{x^{\lambda \eta}} \ln x
    \leq \frac{1}{e \eta \lambda} \frac{1}{x^{(1-\lambda) \eta}}\;,
\end{align*}
it suffices to find $x^\star$ such that $( c e \eta \lambda )^{-1} \leq (x^\star)^{\eta (1 - \lambda)}$, that is,
\begin{align*}
    x^\star
    \geq
    \frac{1}{\eta (1 - \lambda)} \ln \frac{1}{ce\eta\lambda}\;.
\end{align*}
Finally we set $\lambda$ to $1-(ce\eta)^{1/\eta}$. Note that $1-(ce\eta)^{1/\eta} \in (0, 1)$ due to the assumption that $1 / (\eta e) > c$. Then,
\begin{align*}
    \frac{1}{\eta (1 - \lambda)} \ln \frac{1}{ce\eta\lambda}
    = \frac{1}{\eta (ce\eta)^{1/\eta}} \ln \frac{1}{ce \eta ( 1 - (ce \eta)^{1/\eta} )}
    &=
    \frac{Z_2}{\eta(e \eta)^{1/\eta}} \paren*{ \frac{Z_1}{\varepsilon} }^{1/\alpha} \ln \frac{1}{ce\eta ( 1 - (ce\eta)^{1/\eta} )}\\
    &=
    \frac{Z_2\beta}{\alpha}\paren*{\frac{\beta}{e\alpha}}^{\beta / \alpha}\paren*{ \frac{Z_1}{\varepsilon} }^{1/\alpha} \ln \frac{\beta}{ce\alpha ( 1 - (ce\alpha / \beta)^{\beta/\alpha} )}\;.
\end{align*}
Therefore, $k^\star$ is obtained as
\begin{align*}
    k^\star
    =
    \ceil*{
    \frac{\beta}{\alpha}\paren*{\frac{\beta}{e\alpha}}^{\beta / \alpha}\paren*{ \frac{Z_1}{\varepsilon} }^{1/\alpha} \ln \frac{\beta}{ce\alpha ( 1 - (ce\alpha / \beta)^{\beta/\alpha} )}
    }
    = \tiO\paren*{Z_1^{\frac{1}{\alpha}}\varepsilon^{-\frac{1}{\alpha}}}\;.
\end{align*}
\end{proof}

\begin{lemma}[\textbf{Lemma E.6} in \citet{dann2017unifying}]\label{lemma:llnp-property}
$\llnp(x y) \leq \llnp(x)+\llnp(y)+1$ for all $x, y \geq 0$.
\end{lemma}

\begin{lemma}[\citet{cai2023uncoupled} \textbf{Lemma 3}]\label{lemma:k-alpha-ineq}
Let $0 < \alpha < 1$ and $k \geq\left(\frac{24}{1-\alpha} \ln \frac{12}{1-\alpha}\right)^{\frac{1}{1-\alpha}}$. Then,
$k^{1-\alpha}\geq 12 \ln k$.
\end{lemma}

\begin{lemma}\label{lemma:sum-iprod1-j-ineq}
Let $0 < \alpha < 1$, $0 \leq \beta \leq 2$, $c \in \brace*{0}\cup\N$, and let $k \geq\left(\frac{24}{1-\alpha} \ln \frac{12}{1-\alpha}\right)^{\frac{1}{1-\alpha}}$. Then,
\begin{equation*}
\sum_{i=1}^k\left((i+c)^{-\beta} \prod_{j=i+1}^k\left(1-(j+c)^{-\alpha}\right)\right) \leq 9 \ln (k+c) (k+c)^{-\beta+\alpha}
\end{equation*}
\end{lemma}
\begin{proof}
The case when $c=0$ is equivalent to \textbf{Lemma 1} in \citet{cai2023uncoupled}.
For $c\geq 1$, we have
\begin{equation*}
\begin{aligned}
\sum_{i=1}^k\left((i+c)^{-\beta} \prod_{j=i+1}^k\left(1-(j+c)^{-\alpha}\right)\right) 
&=
\sum_{i=1+c}^{k+c}\left(i^{-\beta} \prod_{j=i+1}^{k+c}\left(1-j^{-\alpha}\right)\right) \\
&\leq 
\sum_{i=1}^{k+c}\left(i^{-\beta} \prod_{j=i+1}^{k+c}\left(1-j^{-\alpha}\right)\right) 
\leq 9 \ln (k+c) (k+c)^{-\beta+\alpha}\;,
\end{aligned}
\end{equation*}
where the last inequality uses the result when $c=0$ with replacing $k$ by $k+c$.
\end{proof}

\begin{lemma}\label{lemma:max-iprod-j-ineq}
Let $0 < \alpha < 1$, $0 \leq \beta \leq 2$, $c \in \brace*{0}\cup\N$, and let $k \geq\left(\frac{24}{1-\alpha} \ln \frac{12}{1-\alpha}\right)^{\frac{1}{1-\alpha}}$. Then,
\begin{equation*}
\max_{1 \leq i \leq k}
\left((i+c)^{-\beta} \prod_{j=i+1}^k\left(1-(j+c)^{-\alpha}\right)\right) \leq 4 (k+c)^{-\beta}
\end{equation*}
\end{lemma}
\begin{proof}
The case when $c=0$ is equivalent to \textbf{Lemma 2} in \citet{cai2023uncoupled}.
For $c\geq 1$, we have
\begin{equation*}
\begin{aligned}
\max_{1 \leq i \leq k}
\left((i+c)^{-\beta} \prod_{j=i+1}^k\left(1-(j+c)^{-\alpha}\right)\right) 
&\leq 
\max_{1 + c \leq i \leq k + c}
\left((i+c)^{-\beta} \prod_{j=i+1}^{k+c}\left(1-(j+c)^{-\alpha}\right)\right) \\
&\max_{1 \leq i \leq k + c}
\left((i+c)^{-\beta} \prod_{j=i+1}^{k+c}\left(1-(j+c)^{-\alpha}\right)\right) 
\leq 4 (k+c)^{-\beta}\;.
\end{aligned}
\end{equation*}
where the last inequality uses the result when $c=0$ with replacing $k$ by $k+c$.
\end{proof}

\section{Proof of \cref{lemma:strong duality}}\label{sec:proof of strong-duality}

This section provides the proof of \cref{lemma:strong duality}.
While the proof is a direct modification of \textbf{Lemma 6} in \citet{ding2023last} to the finite-horizon setting, we include the proof here for completeness.
The following lemma is the restatement of \cref{lemma:strong duality}.

\begin{lemma}
For any $\tau \in (0, \infty)$, there exists a unique saddle point $(\pi^\star_\tau, \lambda^\star_\tau) \in \Pi\times \R_+^N$ such that 
$$
L_\tau(\pi^\star_\tau, \lambda) \geq 
L_\tau(\pi^\star_\tau, \lambda^\star_\tau) \geq 
L_\tau(\pi, \lambda^\star_\tau)
\;\; \forall (\pi, \lambda) \in \Pi\times \R^N_+
\;.$$
\end{lemma}
\begin{proof}
Let $\bW \df \brace*{w^\pi[P] \given \pi \in \Pi}$ be the set of all the occupancy measures on $P$.
Let $\widebar{L}_\tau: \bW \times \R^N_+ \to \R$ be the regularized Lagrange function in terms of occupancy measure such that
\begin{equation}\label{eq:occupancy-Lagrange}
\begin{aligned}
&\widebar{L}_\tau(w, \lambda) = \sum_{h, x, a \in \HXA} w_h(x, a) \paren*{r^0_h(x, a) + \sum^N_{n=1} \lambda^n r^n_h(x,a) - \frac{b^n}{H}} + \tau\cH(w) + \frac{\tau}{2}\norm{\lambda}^2_2\\
&\text{where }
\cH(w) \df - \sum_{h, x, a \in \HXA} w_h(x, a) \ln \frac{w_h(x, a)}{\sum_{a' \in \A}w_h(x, a')}
\end{aligned}
\end{equation}
For this Lagrange function, $\widebar{L}_\tau(w^\pi[P], \lambda) = L_\tau(\pi, \lambda)$ holds for any $\pi$ and $\lambda$.
From the one-to-one correspondence between policies and occupancy measures, it is sufficient to prove that there exists a unique saddle point $(w^\star_\tau, \lambda^\star_\tau) \in \bW \times \R^N_+$ such that 
$$
\widebar{L}_\tau(w^\star_\tau, \lambda) \geq 
\widebar{L}_\tau(w^\star_\tau, \lambda^\star_\tau) \geq 
\widebar{L}_\tau(w, \lambda^\star_\tau)
\;\; \forall (w, \lambda) \in \bW\times \R^N_+
\;.$$
Note that $\bW$ is convex and compact \citep{borkar1988convex}.
Furthermore, for any $w \in \bW$, there exists some finite $\lambda_w \in \R^N_+$ such that $\widebar{L}_\tau(w, \lambda_w) = \min_{\lambda \in \R_+^N} \widebar{L}_\tau(w, \lambda)$ due to the regularization $\frac{\tau}{2}\norm{\lambda}^2_2$.
Thus, according to Sion's minimax theorem, the claim immediately holds by showing that $\widebar{L}_\tau(w, \lambda)$ is strictly concave in $w$ and strictly convex in $\lambda$.

It is obvious that $\widebar{L}_\tau(w, \lambda)$ is strictly convex in $\lambda$.
We then show that $\widebar{L}_\tau(w, \lambda)$ is strictly concave in $w$.
According to \cref{eq:occupancy-Lagrange}, it is sufficient to show that $\cH(w)$ is strictly convex in $w$.
For any $w^1, w^2 \in \bW$ and $\alpha \in [0, 1]$, we have
\begin{equation*}
\begin{aligned}
& \cH(\alpha w^1 + (1-\alpha)w^2)\\
= & - \sum_{h, x, a \in \HXA} \paren*{\alpha w_h^1(x, a) + (1-\alpha) w_h^2(x, a)} \ln \frac{\alpha w_h^1(x, a) + (1-\alpha) w_h^2(x, a)}{\sum_{a' \in \A}\alpha w_h^1(x, a') + (1-\alpha) w_h^2(x, a')} \\
\numeq{\geq}{a} & - \sum_{h, x, a \in \HXA} \alpha w_h^1(x, a) \ln \frac{\alpha w^1_h(x, a)}{\sum_{a' \in \A}\alpha w^1_h(x, a')}
- \sum_{h, x, a \in \HXA} (1-\alpha) w_h^2(x, a) \ln \frac{(1-\alpha)w_h^2(x, a)}{\sum_{a' \in \A}(1-\alpha)w_h^2(x, a')}\\
= & \alpha\cH(w^1) + (1-\alpha) \cH(w^2)\;,
\end{aligned}    
\end{equation*}
where (a) is due to the log-sum inequality 
$\left(\sum_i a_i\right) \ln \frac{\sum_i a_i}{\sum_i b_i} \leq \sum_i a_i \ln \frac{a_i}{b_i}$ for non-negative $a_i$ and $b_i$.
Note that the equality of the log-sum inequality holds if and only if $\frac{a_i}{b_i}$ are equal for all $i$.
Therefore, when $w^1 \neq w^2$, we have
$$
\cH(\alpha w^1 + (1-\alpha)w^2)
> 
\alpha \cH(w^1) + (1-\alpha)\cH(w^2)\;.
$$
This concludes the proof.
\end{proof}
\section{Proofs for \cref{theorem:algorithm-is-uniform-PAC}}\label{sec:uniform-PAC-proof}
\subsection{Failure Events and Their Probabilities}
In this section, we use a refined notation of the bonus function:
$\beta_h^{k, \delta} (x, a) = \sum_{y \in \X} \beta_h^{k, \delta} (x, a, y)$, where
\begin{align*}
    \beta_h^{k, \delta} (x, a, y)
    = 
    2 \sqrt{
        \widebar{P}^k_h \paren*{y \given x, a}
    } \phi \paren*{ n^k_h(x, a) } + 5 \phi \paren*{ n^k_h(x, a) }^2\;,
\end{align*}
and
\begin{align*}
    \phi (n)
    =
    1 \wedge \sqrt{
        \frac{2}{n \vee 1} \paren*{
            \llnp(2 n) + \ln \frac{48 \aX^2 \aA H}{\delta}
        }
    }\;.
\end{align*}

For any $\delta > 0$, we define the following failure events.
\begin{equation*}
\begin{aligned}
\cF^P_{\delta}
&\df
\brace*{
    \exists x, y, a, h, k
    :
    \abs*{
        \widebar{P}^k_h \paren*{y \given x, a} - P_h \paren*{y \given x, a}
    }
    \geq
    \beta_h^{k, \delta} (x, a, y)
}\;,
\\
\cF_{\delta}^N&\df\left\{\exists x, a, h, k: n_h^{k}(x, a) < \frac{1}{2}\sum_{j < k} \occ{\pi^j}_h[P](x, a) - \ln \frac{4\aXA H}{\delta}\right\}\;,\\
\cF_{\delta}^{L1}&\df\brace*{\exists x, a, h, k:\norm*{\widebar{P}^k_h\left(\cdot \mid x, a\right)-P_h\left(\cdot \mid x, a\right)}_1 \geq 
\sqrt{\frac{4}{n_{h}^k(x, a)\vee 1}\left(2 \llnp\left(n_{h}^k(x, a)\right)+\ln \frac{12 \aX \aA H(2^{\aX} - 2)}{\delta}\right)}} \;,
\end{aligned}
\end{equation*}
and
$
\cF_\delta
\df
\cF_\delta^P \cup \cF_\delta^N \cup \cF_\delta^{L1}\;,
$
for which the following results hold. 
\begin{lemma}\label{lemma:good event}
    For any $\delta$, the failure probabilities are bounded as follows:
    \begin{align*}
        \P\paren*{ \cF_\delta^P } \leq \frac{\delta}{2}
        \;,\;
        \P\paren*{ \cF_\delta^N } \leq \frac{\delta}{4\aX}
        \;,\;
        \P\paren*{ \cF_\delta^{L1} } \leq \frac{\delta}{4\aX}
        \;,\;
        \text{and }
        \P\paren*{ \cF_\delta } \leq \delta\;,
    \end{align*}
\end{lemma}
\begin{proof}
The bound for $\cF^P_\delta$ holds by a direct application of \textbf{Lemma 6} from \citet{dann2019policy}.
The bounds for $\cF^N_\delta$ and $\cF^{L_1}_\delta$ hold by 
\textbf{Corollary E.4} and \textbf{Corollary E.3} from \citet{dann2017unifying}, respectively.

Accordingly,
\begin{equation*}
   \P(\cF_\delta) \leq 
   \P\paren*{ \cF^P_{\delta}} 
   + \P\paren*{\cF^N_{\delta}} 
   + \P\paren*{\cF^{L1}_{\delta}} 
   \leq \delta\;.
\end{equation*}
\end{proof}

\subsection{Bounds for Policy Estimation}

\begin{lemma}[Policy Estimation Optimism]\label{lemma:estimation-optimism}
Assume that $\cF_\delta^c$ holds.    
For any $(h, x, a) \in \HXA$ and for any episode $k$, the following bound holds
\begin{equation*}
\begin{aligned}
&\tqf{k,n}_h(x, a) - r^n_h(x, a) - P_h \tvf{k,n}_{h+1}(x, a) \geq 0\quad \forall n \in \brace{0} \cup [N]\\
\text{ and }&
\tqf{k}_h(x, a) - r^0_h(x, a) - \sum_{n=1}^N\lambda^{k,n}r^n_h(x, a) - P_h \tvf{k}_{h+1}(x, a) \geq 0\;.
\end{aligned}
\end{equation*}
\end{lemma}
\begin{proof}
For $n = 0$, we have
\begin{equation}\label{eq:optimism-tmp1}
\begin{aligned}
&\tqf{k,0}_h(x, a) - r^0_h(x, a) - P_h \tvf{k,0}_{h+1}(x, a)\\
= &
\min\brace*{r_h^0(x, a) + (1+\tau_k\ln\aA)H\beta^k_h(x, a) + \barP_h \tvf{k,0}_{h+1}(x, a),\; (1+\tau_k\ln\aA)(H-h+1)}
- r_h^0(x, a) - P_h \tvf{k,0}_{h+1}(x, a)\\
= &
\min\brace*{(1+\tau_k\ln\aA)H\beta^k_h(x, a) + \paren*{\barP_h - P_h} \tvf{k,0}_{h+1}(x, a),\;
(1+\tau_k\ln\aA)(H-h+1) - r_h^0(x, a) - P_h \tvf{k,0}_{h+1}(x, a)}\\
\numeq{\geq}{a} &
\min\brace*{(1+\tau_k\ln\aA)H\beta^k_h(x, a) + \paren*{\barP_h- P_h} \tvf{k,0}_{h+1}(x, a),\; 0}\\
\geq &
\min\brace*{
\sum_{y \in \X}(1+\tau_k\ln\aA)H\beta^k_h(x, a, y) - \abs*{\barP_h(y\mid x, a) - P_h(y\mid x, a)}\abs*{\tvf{k,0}_{h+1}(y)},\; 0}\\
\numeq{\geq}{b} &
\min\brace*{
(1+\tau_k\ln\aA)H\sum_{y \in \X}\paren*{\beta^k_h(x, a, y) - \abs*{\barP_h(y\mid x, a) - P_h(y\mid x, a)}},\; 0}
\numeq{\geq}{c} 0\;,
\end{aligned}    
\end{equation}
where (a) is due to $r_h^0(x, a) + P_h \tvf{k,0}_{h+1}(x, a) \leq 1 + (1+\tau_k\ln\aA)\paren*{H - (h+1) + 1} = (1+\tau_k\ln\aA)(H -h +1)$, (b) is due to $\abs*{\tvf{k,0}_{h+1}(y)} \leq (1+\tau_k\ln\aA)H$, and (c) is due to the good event $\cF^c$.
Similarly, for $n \in [N]$, it is easy to verify that
\begin{equation}\label{eq:optimism-tmp2}
\begin{aligned}
&\tqf{k,n}_h(x, a) - r^n_h(x, a) - P_h \tvf{k,n}_h(x, a)\\
\geq &
\min\brace*{
H
\sum_{y \in \X}\paren*{\beta^k_h(x, a, y) - \abs*{\barP_h(y\mid x, a) - P_h(y\mid x, a)}},\; 0}
\geq 0\;.
\end{aligned}
\end{equation}
The first claim holds by \cref{eq:optimism-tmp1} and \cref{eq:optimism-tmp2}.

According to the definition of $\tqf{k}$ in \cref{algo:primal-dual}, we have
\begin{equation}\label{eq:optimism-tmp}
\begin{aligned}
&\tqf{k}_h(x, a) - r^0_h(x, a) - \sum_{n=1}^N\lambda^{k,n}r^n_h(x, a) - P_h \tvf{k}_{h+1}(x, a)\\
=&{\tqf{k,0}_h(x, a) - r^0_h(x, a) - P_h \tvf{k,0}_{h+1}(x, a)}
+\sum^N_{n=1}\lambda^{k,n}\paren*{{\tqf{k,n}_h(x, a) - r^n_h(x, a) - P_h \tvf{k,n}_{h+1}(x, a)}}\;.
\end{aligned}    
\end{equation}
The second claim holds by inserting \cref{eq:optimism-tmp1} and \cref{eq:optimism-tmp2} into \cref{eq:optimism-tmp}.
\end{proof}

The following \emph{nice-episode} technique from \citet{dann2017unifying} is useful to derive the estimation error bound (\cref{lemma:estimation error bound}).

\begin{definition}[$\varepsilon$-Nice Episode]\label{definition: nice-episode}
For $\varepsilon> 0$, let $\wmin(\varepsilon) \df \frac{\varepsilon}{H\Hent\aXA}$.
An episode $k$ is $\varepsilon$-nice if and only if for all $h, x, a \in \HXA$, the following two conditions hold:
\begin{equation*}
w_h^k[P](x, a) \leq \wmin(\varepsilon) \quad\vee\quad n_h^k(x, a) \geq \frac{1}{4}\sum_{i<k} w_h^i(x, a)\;.
\end{equation*}
We also define a set 
$$\bU_h^k(\varepsilon) \df \brace*{(x, a)\in \XA \given w_h^k[P](x, a) \geq \wmin(\varepsilon)}\;.$$
\end{definition}

\begin{lemma}[\textbf{Lemma E.2} in \citet{dann2017unifying}]\label{lemma:number of not-nice episode}
On the good event $\cF_\delta^c$, the number of episodes that are not $\varepsilon$-nice is at most
\begin{equation*}
\frac{6\aX^2\aA H^3}{\varepsilon} \ln \frac{4H\aX\aA}{\delta}\;.    
\end{equation*}
\end{lemma}

\begin{lemma}[\textbf{Lemma E.3} in \citet{dann2017unifying}]\label{lemma:llnp-n-bound}
Fix $r \geq 1$, $\varepsilon > 0$, $C>0$ and $D\geq 1$. 
$C$ may depend polynomially on $\varepsilon^{-1}$ and relevant quantities of $X, A, H, \delta^{-1}$.
$D$ may depend poly-logarithmically on relevant quantities.
Then,
$$
\sum_{h \in [H]} \sum_{x, a \in \bU_{h}^k(\varepsilon)} w_h^k[P](x, a)\left(\frac{C\left(\llnp \left(n_{h}^k(x, a)\right)+D\right)}{n_{h}^k(x, a)}\right)^{1 / r} \leq \varepsilon
$$
on all but at most
\begin{equation*}
\frac{C \aX \aA H^r}{\varepsilon^ r} 
\operatorname{polylog}\left(\aX, A, H, \delta^{-1}, \varepsilon^{-1}\right)
\end{equation*}
$\varepsilon$-nice episodes.
\end{lemma}

\begin{lemma}[Estimation Error Bound]\label{lemma:estimation error bound}
Assume that the good event $\cF^c_\delta$ holds.
For any $k \in \N$ and $\varepsilon > 0$, it holds that 
\begin{equation*}
\begin{aligned}
&\tvf{k,0}_1(x_1) - \vf{\pi^k}_{1}[P, r^0,\tau_k](x_1) \leq \varepsilon\\
\text{and }\;& \tvf{k,n}_1(x_1) - \vf{\pi^k}_1[P, r^n](x_1) \leq \varepsilon \quad \forall n \in [N]
\end{aligned}    
\end{equation*}
on all episodes $k \in \N$ except at most
\begin{equation*}
\frac{\aX^2 \aA H^4}{\varepsilon^{2}}\polylog(\aX, \aA, H, \delta^{-1}, \varepsilon^{-1})
+ \frac{\aX\aA H^3}{\varepsilon}
\polylog(\aX, \aA, H, \delta^{-1}, \varepsilon^{-1})
\end{equation*}
episodes.
\end{lemma}
\begin{proof}
Consider $n=0$.
Note that 
$\tvf{k,0}_h = \pi^k_h\paren*{\tqf{k,0}_h - \tau_k \ln\pi^k_h}$.
Using \cref{lemma:extended value difference}, we have
\begin{equation*}
\begin{aligned}
&\tvf{k,0}_1(x_1)
- \vf{\pi^k}_{1}\brack*{P, r^0, \tau_k}(x_1)\\
=&
\sum_{h=1}^H\sum_{x, a \in \XA} 
w_h^{\pi^\star_{\tau_k}}[P](x) 
\paren*{
\paren*{\tqf{k}_h(x, a) - \tau_k \ln \pi^k_h(x, a)}
- \paren*{r^0_h(x, a) - \tau_k \ln \pi^k_h(x, a)}
- \paren*{P_h \tvf{k}_h}(x, a)
}\\
=&
\sum_{h=1}^H\sum_{x, a \in \XA} 
w_h^{\pi^\star_{\tau_k}}[P](x) 
\paren*{
\tqf{k}_h(x, a)
- r^0_h(x, a)
- \paren*{P_h \tvf{k}_h}(x, a)
}
\geq 0
\;,  
\end{aligned}
\end{equation*}
where the last inequality is due to \cref{lemma:estimation-optimism}.

Accordingly, we have
\begin{equation*}
\begin{aligned}
0\leq
&\tvf{k,0}_1(x_1) - \vf{\pi^k}_{1}\brack*{P, r^0, \tau_k}(x_1)\\
=&
\sum^H_{h=1} \sum_{x, a\in \XA} w^{\pi_k} _h[P](x, a) 
\paren*{
\tqf{k,0}_h(x, a) - r^n_h(x, a) - P_h \tvf{k,0}_{h+1}(x, a)
}\\
=&
\sum^H_{h=1} \sum_{x, a\notin \bU_h^k(\frac{\varepsilon}{3})} 
\underbrace{w^{\pi_k} _h[P](x, a)}_{\leq \wmin(\frac{\varepsilon}{3})}
\underbrace{
\paren*{
\tqf{k,0}_h(x, a) - r^n_h(x, a) - P_h \tvf{k,0}_{h+1}(x, a)
}}_{\leq \Hent}\\
&+ 
\sum^H_{h=1} \sum_{x, a\in \bU_h^k(\frac{\varepsilon}{3})} w^{\pi_k} _h[P](x, a) 
\paren*{
\tqf{k,0}_h(x, a) - r^n_h(x, a) - P_h \tvf{k,0}_{h+1}(x, a)
}\\
\leq&
\underbrace{H\Hent\aXA\wmin\paren*{\frac{\varepsilon}{3}}}_{= \frac{\varepsilon}{3} \text{ due to $\wmin(\frac{\varepsilon}{3})$}}
+ 
\underbrace{
\sum^H_{h=1} \sum_{x, a\in \bU_h^k(\frac{\varepsilon}{3})} w^{\pi_k} _h[P](x, a) 
\paren*{
\tqf{k,0}_h(x, a) - r^n_h(x, a) - P_h \tvf{k,0}_{h+1}(x, a)
}}_{\fd \diamondsuit^k}\;.
\end{aligned}
\end{equation*}
where $\bU^k_h(\frac{\varepsilon}{3})$ is defined in \cref{definition: nice-episode}.

According to the definition of $\tqf{k,n}$, $\diamondsuit^k$ is bounded as
\begin{equation*}
\begin{aligned}
\diamondsuit^k \leq
\sum^H_{h=1} \sum_{x, a\in \bU_h^k(\frac{\varepsilon}{3})} w^{\pi_k} _h[P](x, a) 
\paren*{
\Hent\beta^k_h(x, a) 
+ 
\abs*{\paren*{\barP_h^k - P_h} \tvf{k,0}_{h+1}(x, a)}}\;.
\end{aligned}
\end{equation*}

On the good event $\cF^{L1}_k$, H\"{o}lder's inequality indicates that
\begin{equation*}
\begin{aligned}
\abs*{\paren*{\barP_h^k - P_h} \tvf{k,0}_{h+1}(x, a)}
&\leq 
\norm*{\paren*{\barP_h^k(\cdot\mid x, a) - P_h(\cdot \mid x, a)}}_1\norm*{\tvf{k,0}_{h+1}}_\infty\\
&\leq 
\Hent\sqrt{\frac{4}{n_{h}^k(x, a)\vee 1}\left(2 \llnp\left(n_{h}^k(x, a)\right)+\ln \frac{12 \aX \aA H(2^{\aX} - 2)}{\delta}\right)}\\
&\leq 
\sqrt{\frac{8\Hent^2\aX}{n_{h}^k(x, a)\vee 1}\left( \llnp\left(n_{h}^k(x, a)\right)+\frac{1}{2}\ln \frac{24 \aX \aA H}{\delta}\right)}\;.
\end{aligned}    
\end{equation*}

Also, the bonus term is bounded as
\begin{equation*}
\begin{aligned}
\beta^k_h(x, a)
\leq
&  
\frac{5\aX}{n_{h}^k(x, a)\vee 1}\left( 2\llnp\left(2n_{h}^k(x, a)\right)+2\ln \frac{48 \aX^2 \aA H}{\delta}\right)\\
&+ 
\sum_{y \in \aX}\sqrt{\frac{4 \barP_h\left(y \mid x, a\right)}{n_{h}^k(x, a)\vee 1}\left(2\llnp\left(2n_{h}^k(x, a)\right)+2\ln \frac{48 \aX^2 \aA H^2}{\delta}\right)}\\
\leq
& 
\frac{10\aX}{n_{h}^k(x, a)\vee 1}\left( \llnp\left(n_{h}^k(x, a)\right)+1 + \ln \frac{48 \aX^2 \aA H}{\delta}\right)\\
&+ 
\sqrt{\frac{8 \aX}{n_{h}^k(x, a)\vee 1}\left(\llnp\left(n_{h}^k(x, a)\right)+1+\ln \frac{48 \aX^2 \aA H}{\delta}\right)}\\
\leq& 
\frac{10\aX}{n_{h}^k(x, a)\vee 1}\left( \llnp\left(n_{h}^k(x, a)\right) + 2\ln \frac{48 \aX^2 \aA H}{\delta}\right)
+ 
\sqrt{\frac{8 \aX}{n_{h}^k(x, a)\vee 1}\left(\llnp\left(n_{h}^k(x, a)\right)+2\ln \frac{48 \aX^2 \aA H}{\delta}\right)}\;,
\end{aligned}
\end{equation*}
where the second inequality is due to \cref{lemma:sum of sqrt} and \cref{lemma:llnp-property}.

Accordingly, we have
\begin{equation*}
\begin{aligned}
\diamondsuit^k
\leq
& \sum^H_{h=1} \sum_{x, a\in \bU_h^k(\frac{\varepsilon}{3})} w^{\pi_k} _h[P](x, a) 
\paren*{
\frac{10\aX}{n_{h}^k(x, a)\vee 1}\left( \llnp\left(n_{h}^k(x, a)\right)+2\ln \frac{48 \aX^2 \aA H}{\delta}\right)
}\\
&+ 
\sum^H_{h=1} \sum_{x, a\in \bU_h^k(\frac{\varepsilon}{3})} w^{\pi_k} _h[P](x, a)
\sqrt{\frac{35\Hent^2\aX}{n_{h}^k(x, a)\vee 1}\left(\llnp\left(n_{h}^k(x, a)\right)+2\ln \frac{48 \aX^2 \aA H}{\delta}\right)}\;,
\end{aligned}
\end{equation*}
where we used $\sqrt{8} + \sqrt{8\Hent^2} \leq \sqrt{16 + 16\Hent^2}\leq \sqrt{35\Hent^2}$ due to $(a+b)^2 \leq 2a^2 + 2b^2$ for $a, b\in \R$.

For the first term, we apply \cref{lemma:llnp-n-bound} with $r=1$, $C=10\aX$, and $D=2\ln \frac{48 \aX^2 \aA H}{\delta}$ to bound this term by $\frac{\varepsilon}{3}$ on all but at most 
$$
\frac{\aX^2\aA H}{\varepsilon}\polylog(\aX, \aA, H, \delta^{-1}, \varepsilon^{-1})
$$
nice episodes.

For the second term, we apply \cref{lemma:llnp-n-bound} with $r=2$, $C=35\Hent^2\aX$, and $D=2\ln \frac{48 \aX^2 \aA H}{\delta}$ to bound this term by $\frac{\varepsilon}{3}$ on all but at most 
$$
\frac{\aX^2\aA \Hent^2 H^2}{\varepsilon^{2}}\polylog(\aX, \aA, H, \delta^{-1}, \varepsilon^{-1})
=\frac{\aX^2\aA H^4}{\varepsilon^{2}}\polylog(\aX, \aA, H, \delta^{-1}, \varepsilon^{-1})
$$
nice episodes.

By combining the above results, it holds that 
$\tvf{k,0}_1(x_1) - \vf{\pi^k}_{1}\brack*{P, r^0, \tau_k}(x_1) \leq \varepsilon$
on all $\varepsilon$-nice episodes $k \in \N$ except at most 
$$
\frac{\aX^2 \aA H^4}{\varepsilon^{2}}\polylog(\aX, \aA, H, \delta^{-1}, \varepsilon^{-1})
$$
nice episodes.
Due to \cref{lemma:number of not-nice episode}, the number of not $\varepsilon$-nice episodes is at most 
$\frac{\aX\aA H^3}{\varepsilon}
\polylog(\aX, \aA, H, \delta^{-1}, \varepsilon^{-1})$.
Therefore, 
$\tvf{k,0}_1(x_1) - \vf{\pi^k}_{1}\brack*{P, r^0, \tau_k}(x_1) \leq \varepsilon$ holds
on all episodes $k \in \N$ except at most 
\begin{equation*}
\frac{\aX^2 \aA H^4}{\varepsilon^{2}}\polylog(\aX, \aA, H, \delta^{-1}, \varepsilon^{-1})
+ \frac{\aX\aA H^3}{\varepsilon}
\polylog(\aX, \aA, H, \delta^{-1}, \varepsilon^{-1})
\end{equation*}
episodes.
This concludes the proof of the first claim.
It is easy to verify that the second claim (for $n \in [N]$) holds using the same proof strategy.
\end{proof}

\subsection{Duality Gap Analysis}\label{subsec:duality-gap-analysis}

Recall the regularized optimistic value function $\tvf{k}$ defined in \cref{eq:tvf-definition}.
We decompose the duality gap at episode $k$ as
\begin{equation*}
\begin{aligned}
0 &\leq L_{\tau_k}(\pi^\star_{\tau_k}, \lambda^k) - L_{\tau_k}(\pi^k, \lambda^\star_{\tau_k}) \\
&= 
L_{\tau_k}(\pi^\star_{\tau_k}, \lambda^k)
- \paren*{\tvf{k}_1(x_1) + \frac{{\tau_k}}{2} \|\lambda^k\|_2^2}
+ \paren*{\tvf{k}_1(x_1) + \frac{{\tau_k}}{2} \|\lambda^k\|_2^2}
- L_{\tau_k}(\pi^k, \lambda^\star_{\tau_k})\\
&=
\underbrace{
\vf{\pi^{\star}_{\tau_k}}_{1}\brack*{P, r^0, \tau_k}(x_1)
+ \sum_{n=1}^N\lambda^{k,n} \paren*{\vf{\pi^\star_{\tau_k}}_{1}\brack*{P, r^n}(x_1) - b^n}
- \tvf{k}_1(x_1)
+\sum_{n=1}^N\lambda^{k,n} b^n
}_{\clubsuit^k} 
\\
& \quad 
\underbrace{
-\sum_{n=1}^N\lambda^{k,n} b^n
+ \tvf{k}_1(x_1)
- 
\vf{\pi^{k}}_{1}\brack*{P, r^0, \tau_k}(x_1)
- \sum_{n=1}^N\lambda^{\star,n}_{\tau_k} \paren*{\vf{\pi^k}_{1}\brack*{P, r^n}(x_1) - b^n}
+ \frac{\tau_k}{2} \|\lambda^k\|_2^2 - \frac{\tau_k}{2}\|\lambda^\star_{\tau_k}\|_2^2}_{\heartsuit^k}\;.
\end{aligned}
\end{equation*}

\subsubsection{$\clubsuit^k$ Bound}

Let $\gamma^k_h(x) \df \KL{\pi^\star_{\tau_k,h}\paren*{\cdot \given x}}{{\pi^k_h\paren*{\cdot \given x}}}$.
It is easy to see that
$$
\vf{\pi^\star_{\tau_k}}_1\brack*{P, - \ln \pi^{k}}(x_1)
- \vf{\pi^\star_{\tau_k}}_1\brack*{P, - \ln \pi^\star_{\tau_k}}(x_1)
= 
\sum_{h=1}^H\sum_{x \in \X} 
w_h^{\pi^\star_{\tau_k}}[P](x) 
\gamma^{k}_h(x)\;.
$$

\begin{lemma}
Let 
$C_1 \df 2\Hent^2 \paren*{1 + \frac{H}{\bgap}}^2$
and 
$C_{2,k} \df 2 \paren*{1 + \ln \aA}^2$. 
Assume that the good event $\cF^c_\delta$ holds. Then, 
\begin{equation*}
\clubsuit^k \leq \sum_{h=1}^H\sum_{x \in \X} 
w_h^{\pi^\star_{\tau_k}}[P](x) 
\frac{1}{\eta_{k}}\paren*{
(1 - \eta_{k}\tau_k)\gamma^{k}_h(x) - \gamma^{k+1}_h(x)
} + \eta_{k} C_1 + \eta_k \tau_k^2 C_{2,k}\;.  
\end{equation*}
\end{lemma}
\begin{proof}
Let $g^k \df r^0 - \tau_k \ln \pi^{k} + \sum_{n=1}^N\lambda^{k,n} r^n$.
Note that 
$\vf{\pi}_{1}\brack*{P, r^0, \tau} = 
\vf{\pi}_{1}\brack*{P, r^0} + \tau\vf{\pi}_{1}\brack*{P, -\ln\pi}$ for any $\pi \in \Pi$ and $\tau \geq 0$.
Accordingly,
\begin{equation*}
\begin{aligned}
\clubsuit^k &= 
\vf{\pi^{\star}_{\tau_k}}_{1}\brack*{P, r^0, \tau_k}(x_1)
+ \sum_{n=1}^N\lambda^{k,n} \paren*{\vf{\pi^\star_{\tau_k}}_{1}\brack*{P, r^n}(x_1) - b^n}
- \tvf{k}_1(x_1)
+ \sum_{n=1}^N\lambda^{k,n} b^n
\\
&= 
\vf{\pi^\star_{\tau_k}}_1\brack*{P, r^0 + \sum_{n=1}^N\lambda^{k,n} r^n}(x_1)
- \tvf{k}_1(x_1)
+ \tau_k \vf{\pi^\star_{\tau_k}}_1\brack*{P, - \ln \pi^\star_{\tau_k}}(x_1) \\
&= 
\vf{\pi^\star_{\tau_k}}_1\brack*{ P, g^k}(x_1)
- \tvf{k}_1(x_1)
+ \tau_k \vf{\pi^\star_{\tau_k}}_1\brack*{P, - \ln \pi^\star_{\tau_k}}(x_1)
- \tau_k \vf{\pi^\star_{\tau_k}}_1\brack*{P, - \ln \pi^k}(x_1) \\
&= 
\underbrace{
\vf{\pi^\star_{\tau_k}}_1\brack*{ P, g^k}(x_1)
- \tvf{k}_1(x_1)
}_{\diamondsuit}
- \tau_k\sum_{h=1}^H\sum_{x \in \X} 
w_h^{\pi^\star_{\tau_k}}[P](x) 
\gamma^{k}_h(x)\;.
\end{aligned}
\end{equation*}

Using the definition of $\tqf{k}$ in \cref{algo:primal-dual}, we have
\begin{equation*}
\tvf{k}_h
= \tvf{k,0}_h +  \sum_{n=1}^N\lambda^{k,n}\tvf{k,n}_h
= \pi^k_h\paren*{\tqf{k,0}_h
- \tau_k\ln \pi^k_h
+  \sum_{n=1}^N\lambda^{k,n}\tqf{k,n}_h}
= \pi^k_h\paren*{\tqf{k}_h - \tau_k \ln\pi^k_h}
\;.
\end{equation*}
Using \cref{lemma:extended value difference}, we have

\begin{equation}\label{eq:diamond temp1}
\begin{aligned}
- \diamondsuit
&=
\tvf{k}_1(x_1)
- \vf{\pi^\star_{\tau_k}}_1\brack*{ P, g^k}(x_1)\\
&=
\sum_{h=1}^H\sum_{x, a \in \XA} 
w_h^{\pi^\star_{\tau_k}}[P](x) 
\paren*{
\pi^\star_{\tau_k,h}\paren*{a \given x} - \pi^k_h\paren*{a \given x} 
}
\paren*{\tqf{k}_h(x, a) - \tau_k \ln \pi^k_h(x, a)}\\
&\quad +
\underbrace{\sum_{h=1}^H\sum_{x, a \in \XA} 
w_h^{\pi^\star_{\tau_k}}[P](x, a) 
\paren*{
\tqf{k}_h(x, a)
- \tau_k \ln \pi^k_h(x, a)
- g^k_h(x, a)
- \paren*{P_h \tvf{k}_h}(x, a)
}}_{\geq 0 \text{ due to \cref{lemma:estimation-optimism}}}\\
&\geq 
\sum_{h=1}^H\sum_{x, a \in \XA} 
w_h^{\pi^\star_{\tau_k}}[P](x) 
\paren*{
\pi^\star_{\tau_k,h}\paren*{a \given x} - \pi^k_h\paren*{a \given x} 
}
\paren*{\tqf{k}_h(x, a)
- \tau_k \ln \pi^k_h(x, a)}\;.
\end{aligned}
\end{equation}

Note that $\pi^{k+1}$ is the closed-form solution of the KL-regularized greedy policy (e.g., \textbf{Equation (5)} of \citet{kozuno2019theoretical}):
\begin{equation}\label{eq:KL regularized policy argmin}
\begin{aligned}
\pi^{k+1}_h(\cdot \mid x) &\propto \pi^k_h(\cdot \mid x) \exp\paren*{\eta_k \paren*{\tqf{k}_h - \tau_k\ln\pi^k_h}(x, \cdot)} \\
&=
\argmin_{\widetilde{\pi} \in \Delta_\A} \brace*{\sum_{a \in \A} \widetilde{\pi}(a) \paren*{\paren*{-\tqf{k}_h+ \tau_k\ln\pi^k_h}(x, a)} + \frac{1}{\eta_k} \KL{\widetilde{\pi}}{\pi^k_h(\cdot\mid x)}}\;.
\end{aligned}
\end{equation}

Also, due to the definition of $\lambda^k$,
\begin{equation}\label{eq:tqf-upper bound}
\norm*{\tqf{k}}_\infty \leq \underbrace{\norm*{\tqf{k,0}}_\infty}_{\leq \Hent} + \underbrace{\sum^N_{n=1}\lambda^{k,n}}_{\leq \Hent / \bgap}\underbrace{\norm*{\tqf{k,n}}_\infty}_{\leq H}
\leq
\Hent \paren*{1 + \frac{H}{\bgap}}\;.
\end{equation}

Using \cref{lemma:online mirror ascent} with \cref{eq:KL regularized policy argmin}, we have
\begin{equation}\label{eq:diamond temp2}
\begin{aligned}
&\paren*{
\pi^\star_{\tau_k,h}\paren*{a \given x} - \pi^k_h\paren*{a \given x}
}
\paren*{\tqf{k}_h- \tau_k\ln\pi^k_h}(x, a)\\
\leq&
\frac{1}{\eta_k}\paren*{
\gamma^{k}_h(x) - \gamma^{k+1}_h(x)
}
+ 2\eta_{k} \sum_{a\in \A}\pi^k_h(a\mid x)\paren*{\tqf{k}_h(x, a)}^2
+ 2\eta_k \tau^2_k \paren*{1 + \ln \aA}^2 \\
\numeq{\leq}{a}&
\frac{1}{\eta_k}\paren*{
\gamma^{k}_h(x) - \gamma^{k+1}_h(x)
}
+ \eta_k C_1 + \eta_k\tau_k^2 C_{2, k}\;,
\end{aligned}
\end{equation}
where (a) is due to \cref{eq:tqf-upper bound}.
By substituting \cref{eq:diamond temp1} and \cref{eq:diamond temp2} to $\diamondsuit$, we have
$$
\clubsuit^k 
= \diamondsuit - \sum_{h=1}^H\sum_{x \in \X} 
w_h^{\pi^\star_{\tau_k}}[P](x) 
\gamma^{k}_h(x)
\leq 
\sum_{h=1}^H\sum_{x\in \X} 
w_h^{\pi^\star_{\tau_k}}[P](x) 
\frac{1}{\eta_{k}}\paren*{
(1 - \eta_{k}\tau_k)\gamma^{k}_h(x) - \gamma^{k+1}_h(x)
} + \eta_{k} C_1 + \eta_k\tau_k^2 C_{2, k}\;.
$$
\end{proof}

\subsubsection{$\heartsuit^k$ Bound}

\begin{lemma}
Let $\rho^k \df  
\tvf{k,0}_1(x_1) - \vf{\pi^{k}}_{1}\brack*{P, r^0, \tau_k}(x_1)
+ \sum^N_{n=1} \lambda^{\star,n}_{\tau_k} 
\paren*{\tvf{k,n}_1(x_1) - \vf{\pi^k}_{1}\brack*{P, r^n}(x_1)}$.
Let $C_3 \df \frac{N}{2}\paren*{H + \frac{\Hent}{\bgap}}^2$.
Then, 
\begin{equation*}
\heartsuit^k \leq
 \frac{1}{2\eta_{k}}\paren*{\paren*{1 - \eta_{k}\tau_k}\|\lambda^\star_{\tau_k} - \lambda^{k}\|_2^2 - \|\lambda^\star_{\tau_k} - \lambda^{k}\|_2^2}
+ \rho^k + \frac{1}{2}\eta_{k} C_3 \;.  
\end{equation*}
\end{lemma}
\begin{proof}
Recall that
\begin{equation*}
\begin{aligned}
\heartsuit^k
= 
\tvf{k}_1(x_1)
-\sum_{n=1}^N\lambda^{k,n} b^n
- 
\vf{\pi^{k}}_{1}\brack*{P, r^0, \tau_k}(x_1)
- \sum_{n=1}^N\lambda^{\star,n}_{\tau_k} \paren*{\vf{\pi^k}_{1}\brack*{P, r^n}(x_1) - b^n}
+ \frac{\tau_k}{2} \|\lambda^k\|_2^2 - \frac{\tau_k}{2}\|\lambda^\star_{\tau_k}\|_2^2\;.
\end{aligned}
\end{equation*}

Note that
\begin{equation}\label{eq:heart tmp1}
\begin{aligned}
&
\tvf{k}_1(x_1) 
-\sum_{n=1}^N\lambda^{k,n} b^n
- \vf{\pi^{k}}_{1}\brack*{P, r^0, \tau_k}(x_1)
- \sum_{n=1}^N\lambda^{\star,n}_{\tau_k} \paren*{\vf{\pi^k}_{1}\brack*{P, r^n}(x_1) - b^n} \\
=&
\tvf{k,0}_1(x_1) + \sum^N_{n=1} \lambda^{k,n} \paren*{\tvf{k,n}_1(x_1) - b^n}
- \vf{\pi^{k}}_{1}\brack*{P, r^0, \tau_k}(x_1)
- \sum_{n=1}^N\lambda^{\star,n}_{\tau_k} \paren*{\vf{\pi^k}_{1}\brack*{P, r^n}(x_1) - b^n}\\
=&
\underbrace{
\tvf{k,0}_1(x_1) - \vf{\pi^{k}}_{1}\brack*{P, r^0, \tau_k}(x_1)
+ \sum^N_{n=1} \lambda^{\star,n}_{\tau_k} 
\paren*{\tvf{k,n}_1(x_1) - \vf{\pi^k}_{1}\brack*{P, r^n}(x_1)}
}_{= \rho^k}
+ \sum_{n=1}^N\paren*{\lambda^{k,n} - \lambda^{\star,n}_{\tau_k}} \paren*{\tvf{k,n}_1(x_1) - b^n}
\end{aligned}
\end{equation}

Also, note that 
\begin{equation}\label{eq:heart tmp2}
\|\lambda^k\|_2^2 - \|\lambda^\star_{\tau_k}\|_2^2
= \sum_{n=1}^{N} (\lambda^{k,n})^2 - (\lambda^{\star,n}_{\tau_k})^2
= \sum_{n=1}^{N} 2\lambda^{k,n}(\lambda^{k,n} - \lambda^{\star,n}_{\tau_k}) -
(\lambda^{k,n} - \lambda^{\star,n}_{\tau_k})^2
\end{equation}
and
\begin{equation}\label{eq:heart tmp3}
\paren*{\underbrace{\tvf{k,n}_1(x_1) - b^n}_{\in [-H, H]}+ \underbrace{\tau_k\lambda^{k,n}}_{\leq \Hent / \bgap}}^2
\leq \paren*{H + \frac{\Hent}{\bgap}}^2\;.
\end{equation}

Combining \cref{eq:heart tmp1}, \cref{eq:heart tmp2}, and \cref{eq:heart tmp3}, we have
\begin{equation*}
\begin{aligned}
\heartsuit^k \numeq{\leq}{a} 
&\rho^k + \sum_{n=1}^N\paren*{\lambda^{k,n} - \lambda^{\star,n}_{\tau_k}} \paren*{\tvf{k,n}_1(x_1) - b^n} 
+ \sum_{n=1}^{N} \tau_k\lambda^{k,n}(\lambda^{k,n} - \lambda^{\star,n}_{\tau_k}) -
\frac{\tau_k}{2}(\lambda^{k,n} - \lambda^{\star,n}_{\tau_k})^2\\
=&
\rho^k + \sum_{n=1}^N\paren*{\lambda^{k,n} - \lambda^{\star,n}_{\tau_k}} \paren*{\tvf{k,n}_1(x_1) - b^n+ \tau_k\lambda^{k,n}}
- \frac{\tau_k}{2}(\lambda^{k,n} - \lambda^{\star,n}_{\tau_k})^2\\
\numeq{\leq}{b}&
\rho^k 
+\sum_{n=1}^N
\frac{1}{2\eta_k}\paren*{(\lambda^{k,n} - \lambda^{\star,n}_{\tau_k})^2 - (\lambda^{k+1,n} - \lambda^{\star,n}_{\tau_k})^2}
- \frac{\tau_k}{2}(\lambda^{k,n} - \lambda^{\star,n}_{\tau_k})^2
+ \frac{\eta_k}{2} \paren*{H + \frac{\Hent}{\bgap}}^2
\\
\leq&
\frac{1}{2\eta_{k}}\paren*{\paren*{1 - \eta_{k}\tau_k}\|\lambda^\star_{\tau_k} - \lambda^{k}\|_2^2 - \|\lambda^\star_{\tau_k} - \lambda^{k+1}\|_2^2}
+ \rho^k + \eta_k C_3
\end{aligned}
\end{equation*}
where (a) uses \cref{eq:heart tmp2} and (b) uses \cref{lemma:gradient descent} with the definition of $\lambda^{k+1}$ and \cref{eq:heart tmp3}.
\end{proof}

By combining the bounds of $\clubsuit^k$ and $\heartsuit^k$, 
under the good event $\cF^c_\delta$, we have

\begin{equation}\label{eq:duality gap bound}
\begin{aligned}
0 \leq \clubsuit^k + \heartsuit^k &\leq
\sum_{h=1}^H\sum_{x \in \X} 
w_h^{\pi^\star_{\tau_k}}[P](x) 
\paren*{
(1 - \eta_{k}\tau_k)
\gamma^{k}_h(x) - \gamma^{k+1}_h(x)
} + \eta_{k}^2 C_1 + \eta_k^2 \tau_k^2 C_{2, k}\\
& \quad + 
\frac{1}{2}\paren*{\paren*{1 - \eta_{k}\tau_k}\|\lambda^\star_{\tau_k} - \lambda^{k}\|_2^2 - \|\lambda^\star_{\tau_k} - \lambda^{k+1}\|_2^2}
+ \eta_{k}\rho^k + \eta_{k}^2 C_3\;.
\end{aligned}
\end{equation}

\subsection{Optimaility Gap and Constraint Violation Analysis}\label{appendix:optimality-gap-and-vio-analysis}

\looseness=-1
Let 
$\Phi^k \df 
\sum_{h=1}^H\sum_{x \in \X} 
w_h^{\pi^\star_{\tau_k}}[P](x) 
\gamma^k_h(x)
+ \frac{1}{2}\|\lambda^\star_{\tau_k} - \lambda^k\|_2^2$
and $C \df \max_{k \in [N]}\brace*{C_1 + \tau_k^2 C_{2,k} + {C_3}}$.
By rearranging \cref{eq:duality gap bound}, we get
\begin{equation*}
\begin{aligned}
\Phi^{k+1}    
&\leq  (1 - \eta_{k}\tau_k)\Phi^{k} + \eta_{k}^2 C + \eta_{k} \rho^k\\
&\leq  (1 - \eta_{k}\tau_k)(1 - \eta_{k-1}\tau_{k-1})\Phi^{k-1} + \paren*{(1 - \eta_{k}\tau_k) \eta_{k-1}^2 + \eta_k^2} C
+ \paren*{(1 - \eta_{k}\tau_k)\eta_{k-1}\rho^{k-1} + \eta_{k} \rho^k}\\
&\leq \cdots\\
&\leq
A^{k}_1\Phi^{1} + B_{k}C + E_{k}\;,
\end{aligned}
\end{equation*}
where 
$A^{k}_t = \prod_{i=t}^{k}(1 -  \eta_i\tau_i)$,
$B_{k} = \sum_{i=1}^{k} A^{k}_{i+1}\eta_{i}^2$, and
$E_{k} = \sum_{i=1}^{k} A^{k}_{i+1}\eta_{i} \rho^i$.
For $\Phi^{k+1}$, the following lemma holds.

\begin{lemma}\label{lemma:Phi bound}
Set the learning rate and the regularization coefficient as $\eta_k=(k + 3)^{-\alpha_\eta}$ and $\tau_k=(k + 3)^{-\alpha_\tau}$.
Set $\alpha_\tau$ and $\alpha_\eta$ such that
$0 < \alpha_\tau < 0.5 < \alpha_\eta < 1$ and 
$\alpha_\eta + \alpha_\tau < 1$.
Let $k^\star \df \left(\frac{24}{1-(\alpha_\eta + \alpha_\tau)} \ln \frac{12}{1-(\alpha_\eta + \alpha_\tau)}\right)^{\frac{1}{1-(\alpha_\eta + \alpha_\tau)}}$.

Assume that \cref{assumption:slater} and the good event $\cF^c_\delta$ hold.
Then, for any $\varepsilon > 0$, $\Phi^{k+1} \leq  \varepsilon$ is satisfied for any $k \in \N$ except at most 
\begin{equation*}
\begin{aligned}
\tiO\paren*{\paren*{\bgap^{-1}{(1+N)X\sqrt{A}H^{3}\varepsilon^{-1}}}^{\frac{1}{0.5 - \alpha_\tau}}} 
+\tiO\paren*{\paren*{\bgap^{-2}{(1+N)H^{4}\varepsilon^{-1}}}^{\frac{1}{\alpha_\eta - \alpha_\tau}}}
+ k^\star
\end{aligned}
\end{equation*}
episodes.   
\end{lemma}
\begin{proof}
Using \cref{lemma:sum-iprod1-j-ineq}, for $k \geq k^\star$, we have
\begin{equation*}
B_{k} 
= \sum_{i=1}^{k} A^{k}_{i+1}\eta_{i}^2   
= \sum_{i=1}^{k} \eta_{i}^2  \prod_{j=i+1}^{k}(1 -  \eta_j\tau_j)
= \sum_{i=1}^{k} (i+3)^{-2\alpha_\eta} \prod_{j=i+1}^{k}(1 -  (j+3)^{-\alpha_\eta - \alpha_\tau})
\leq
9\ln(k+3)(k+3)^{\alpha_\tau - \alpha_\eta} \;.
\end{equation*}

Note that 
$
\frac{1}{\prod_{j=2}^3\paren*{1 - j^{-\alpha_\eta - \alpha_\tau}}}
\leq \frac{1}{\paren*{1 - 2^{-0.5}}\paren*{1 - 3^{-0.5}}}
\leq 9$.
For $A^k_1$, when $k \geq k^\star$, we have
\begin{equation*}
\begin{aligned}
A^{k}_1 
& = \prod_{i=1}^{k}\paren*{1 -  (i+3)^{-\alpha_\eta - \alpha_\tau}}
= \prod_{i=4}^{k+3}\paren*{1 -  i^{-\alpha_\eta - \alpha_\tau}}
\leq 9 \prod_{i=2}^{k+3}\paren*{1 - i^{-\alpha_\eta - \alpha_\tau}}
\leq 9 \paren*{1 -  (k+3)^{-\alpha_\eta - \alpha_\tau}}^{k+3}\\
&\numeq{\leq}{a} 9 \paren*{\exp\paren*{-(k+3)^{-(\alpha_\eta + \alpha_\tau)}}}^{k+3}
= 9 {\exp\paren*{-(k+3)^{1-(\alpha_\eta + \alpha_\tau)}}}
\numeq{\leq}{b}
9 \exp\paren*{-12\ln (k+3)} = 9 (k+3)^{-12}\;.
\end{aligned}
\end{equation*}
where (a) uses $1 - x \leq \exp(-x)$ and (b) uses \cref{lemma:k-alpha-ineq}.

Using \cref{lemma:max-iprod-j-ineq}, for $k \geq k^\star$,  we have
\begin{equation*}
\max_{1\leq i \leq k} \eta_i A_{i+1}^k
=\max_{1\leq i \leq k} \eta_i \prod_{j=i+1}^{k}(1 -  \eta_j\tau_j)
=\max_{1\leq i \leq k} (i+3)^{-\alpha_\eta} \prod_{j=i+1}^{k}\paren*{1 -  (j+3)^{-\alpha_\eta -\alpha_\tau}}
\leq
4(k+3)^{-\alpha_\eta}\;.
\end{equation*}
This indicates that
\begin{equation*}
E_k 
= \sum_{i=1}^{k} A^{k}_{i+1}\eta_{i} \rho^i
\leq
4(k+3)^{-\alpha_\eta}
\sum_{i=1}^k \rho^i\;.
\end{equation*}

Therefore, $\Phi^{k+1}$ is bounded as
\begin{equation*}
\Phi^{k+1}
\leq
A_1^k\Phi^1 + B_kC + E_k
\leq 
\underbrace{
9 \Phi^1 (k+3)^{-12}
}_{\mathrm{(i)}}
+ \underbrace{\paren*{9C\ln(k+3)}(k+3)^{\alpha_\tau - \alpha_\eta}}_{\mathrm{(ii)}}
+ \underbrace{4(k+3)^{-\alpha_\eta} \sum^k_{i=1} \rho^i}_{\mathrm{(iii)}}\;.
\end{equation*}

\paragraph{$\mathrm{(i)}$ bound.}
Since $\pi^1$ is a uniform policy and $\lambda^1 = \bzero$, \cref{lemma:property-of-regularized} indicates that
\begin{equation*}
\Phi^1
=
\sum_{h=1}^H\sum_{x \in \X} 
w_h^{\pi^\star_{\tau_k}}[P](x) 
\underbrace{\KL{\pi^\star_{\tau_1;h}(\cdot\mid x)}{\pi^1_h(\cdot\mid x)}}_{\leq 2\ln \aA}
+ \frac{1}{2}\underbrace{\|\lambda^\star_{\tau_1} - \lambda^1\|_2^2}_{=\norm*{\lambda^\star_{\tau_1}}_2^2 \leq N\Hent^2/\bgap^2}
\leq
2H\ln A
+ \frac{N\Hent^2}{2\bgap^2}\;.
\end{equation*}
Therefore, for any $\varepsilon > 0$, $\mathrm{(i)}=9 \Phi^1 (k+3)^{-12} \leq \varepsilon$ is satisfied for any $k \in \N$ except at most
$$
\tiO\paren*{
\paren*{
\varepsilon^{-1}
\paren*{
H + \bgap^{-2}{NH^2}}}^{\frac{1}{12}}} + k^\star
$$
episodes.

\paragraph{$(\mathrm{ii})$ bound.}
Recall that 
\begin{equation*}
\begin{aligned}
C &= 
\paren*{1 + \ln \aA}^2
+ \Hent^2 \paren*{1 + \frac{H}{\bgap}}^2
+ \frac{N}{2}\paren*{H + \frac{\Hent}{\bgap}}^2\;.
\end{aligned}
\end{equation*}

Accordingly, we have
\begin{equation*}
9C\ln(k+3) = \paren*{\frac{H^4}{\bgap^2} + \frac{NH^2}{\bgap^2}}\polylog(k)\;.
\end{equation*}
\cref{lemma:k-1 vs polylog k} indicates that, for any $\varepsilon > 0$, $\mathrm{(ii)}=\paren*{9C\ln(k+3)} (k+3)^{\alpha_\tau - \alpha_\eta} \leq \varepsilon$ is satisfied for any $k \in \N$ except at most
$$
\tiO\paren*{\paren*{\varepsilon^{-1}
\paren*{{\bgap^{-2}}{H^4}+ {\bgap^{-2}}{NH^2}}
}^{\frac{1}{\alpha_\eta - \alpha_\tau}}}
 + k^\star
$$
episodes.

\paragraph{$(\mathrm{iii})$ bound.}
\cref{lemma:estimation-optimism} indicates that
\begin{equation*}
0 \leq \sum^k_{i=1}\rho^i
\leq
\sum^k_{i=1} \paren*{\tvf{i,0}_1(x_1) - \vf{\pi^{i}}_{\tau_i;1}\brack*{P, r^0}(x_1)}
+ \frac{\Hent}{\bgap}\sum^N_{n=1}
\sum^k_{i=1}
\paren*{\tvf{i,n}_1(x_1) - \vf{\pi^i}_{1}\brack*{P, r^n}(x_1)}\;.
\end{equation*}
By applying \cref{lemma:error-to-regret} to \cref{lemma:estimation error bound}, we have
\begin{equation*}
\begin{aligned}
&\sum^k_{i=1} \paren*{\tvf{i,0}_1(x_1) - \vf{\pi^{i}}_{\tau_i;1}\brack*{P, r^0}(x_1)}
\leq \sqrt{k X^2AH^4}\polylog\paren*{k, X, A, H, \delta^{-1}}\\
\text{ and }\; &
\sum^k_{i=1} \paren*{\tvf{i,n}_1(x_1) - \vf{\pi^{i}}_{1}\brack*{P, r^n}(x_1)}
\leq \sqrt{k X^2AH^4}\polylog\paren*{k, X, A, H, \delta^{-1}} \quad \forall n \in [N]\;.
\end{aligned}
\end{equation*}
This indicates that 
\begin{equation*}
\sum^k_{i=1}\rho^i
\leq 
\paren*{1 + \frac{N\Hent}{\bgap}}
\sqrt{k X^2AH^4}\polylog\paren*{k, X, A, H, \delta^{-1}}\;.
\end{equation*}
\cref{lemma:k-1 vs polylog k} indicates that, for any $\varepsilon > 0$, $\mathrm{(iii)}=4(k+3)^{-\alpha_\eta} \sum_{i=1}^k \rho^i \leq \varepsilon$ is satisfied for any $k \in \N$ except at most
$$
\tiO\paren*{\paren*{\varepsilon^{-1}{X\sqrt{A}\paren*{H^{2} + \bgap^{-1}NH^3}}}^{\frac{1}{\alpha_\eta - 0.5}}}
+ k^\star
$$
episodes.
By combining the above results, we have
$\Phi^{k+1} \leq \varepsilon$ for all $k \in \N$ except at most
\begin{equation*}
\begin{aligned}
&
\tiO\paren*{
\paren*{
\varepsilon^{-1}
\paren*{
H + \bgap^{-2}{NH^2}}}^{\frac{1}{12}}} 
+ 
\tiO\paren*{\paren*{\varepsilon^{-1}
\paren*{{\bgap^{-2}}{H^4}+ {\bgap^{-2}}{NH^2}}
}^{\frac{1}{\alpha_\eta - \alpha_\tau}}}\\
&+ 
\tiO\paren*{\paren*{\varepsilon^{-1}{X\sqrt{A}\paren*{H^{2} + \bgap^{-1}NH^3}}}^{\frac{1}{\alpha_\eta - 0.5}}}
+ k^\star
\\
=&
\tiO\paren*{\paren*{\varepsilon^{-1}{X\sqrt{A}H^4\bgap^{-2}\paren*{1+N}}}^{\frac{1}{\alpha_\eta - 0.5}}}
+ k^\star
\end{aligned}
\end{equation*}
episodes, where we used $\alpha_\tau < 0.5$ and $\bgap^{-2}H^4 \geq H^2$ due to $\bgap \leq H$.
\end{proof}

\subsection{Proof of \cref{theorem:algorithm-is-uniform-PAC}}
We are now ready to prove the main claim.
Consider $\alpha_\tau$ and $\alpha_\eta$ satisfy conditions specified in \cref{theorem:algorithm-is-uniform-PAC}.
Suppose that the good event $\cF^c_\delta$ holds.

Using $k^\star$ defined in \cref{lemma:Phi bound}, for any $\varepsilon > 0$, we have $\Phi^k \leq \frac{\varepsilon^2}{H^3}$ for any $k \in \N$ except at most
\begin{equation*}
\tiO\paren*{\paren*{\paren*{\frac{\varepsilon^2}{H^3}}^{-1}{X\sqrt{A}H^4\bgap^{-2}\paren*{1+N}}}^{\frac{1}{\alpha_\eta - 0.5}}}
+ k^\star 
=\tiO\paren*{\paren*{\bgap^{-2}\paren*{1+N}{X\sqrt{A}H^7\varepsilon^{-2}}}^{\frac{1}{\alpha_\eta - 0.5}}}
+ k^\star 
\end{equation*}
episodes.

Also, when $k \geq \paren*{\frac{\Hent}{\varepsilon\min\brace{\bgap, 1}}}^{\frac{1}{\alpha_\tau}}$, we have $\tau_k \Hent \leq \varepsilon$ and $\frac{\tau_k\Hent}{\bgap} \leq \varepsilon$.
Furthermore, it is easy to see that $\Phi^k \leq \frac{\varepsilon^2}{H^3}$ indicates $\sum_{h=1}^H\sum_{x \in \X} w_h^{\pi^\star_{\tau_k}}[P](x) \gamma^k_h(x) \leq \frac{\varepsilon^2}{H^3}$.
Then, \cref{lemma:KL-to-optimality} indicates that 
\begin{equation*}
V_1^{\pi^{\star}}[P, r](x_1)-
V_1^{\pi^k}[P, r](x_1)
\leq \varepsilon
\;\text{ and }\;
b^n -
V_1^{\pi^k}[P, r^n](x_1)
\leq \varepsilon \quad \forall n \in [N]
\end{equation*}
hold for any $\varepsilon > 0$ and for any $k \in \N$ except at most
\begin{equation*}
{\tiO\paren*{\paren*{\bgap^{-2}\paren*{1+N}{X\sqrt{A}H^7\varepsilon^{-2}}}^{\frac{1}{\alpha_\eta - 0.5}}}}
+ {\tiO\paren*{\paren*{\varepsilon^{-1}\bgap^{-1}H}^{\frac{1}{\alpha_\tau}}}}
+ {\left(\frac{24}{1-(\alpha_\eta + \alpha_\tau)} \ln \frac{12}{1-(\alpha_\eta + \alpha_\tau)}\right)^{\frac{1}{1-(\alpha_\eta + \alpha_\tau)}}}
\end{equation*}
episodes.

Finally, \cref{lemma:good event} shows that the good event $\cF^c_\delta$ holds with probability at least $1-\delta$.
This concludes the proof of \cref{theorem:algorithm-is-uniform-PAC}.

\end{document}